\documentclass[letterpaper]{article} % DO NOT CHANGE THIS
\pdfoutput=1
\usepackage{aaai24}  % DO NOT CHANGE THIS
\usepackage{times}  % DO NOT CHANGE THIS
\usepackage{helvet}  % DO NOT CHANGE THIS
\usepackage{courier}  % DO NOT CHANGE THIS
\usepackage[hyphens]{url}  % DO NOT CHANGE THIS
\usepackage{graphicx} % DO NOT CHANGE THIS
\usepackage{natbib}  % DO NOT CHANGE THIS AND DO NOT ADD ANY OPTIONS TO IT
\usepackage{caption} % DO NOT CHANGE THIS AND DO NOT ADD ANY OPTIONS TO IT
\usepackage[ruled,linesnumbered]{algorithm2e}

\usepackage{newfloat}
\usepackage{listings}

\usepackage{amssymb}
\usepackage{amsmath}
\usepackage{amsthm}
\usepackage{tabularx}
\usepackage{multirow}
\usepackage{subfigure}
\usepackage{booktabs}
\newtheorem{theorem}{Theorem}

\DeclareCaptionStyle{ruled}{labelfont=normalfont,labelsep=colon,strut=off} % DO NOT CHANGE THIS
\title{Neural Gaussian Similarity Modeling for Differential Graph Structure Learning}
\author {
    Xiaolong Fan\textsuperscript{\rm 1},
    Maoguo Gong \textsuperscript{\rm 1},
    Yue Wu\textsuperscript{\rm 2},
    Zedong Tang\textsuperscript{\rm 3}, and
    Jieyi Liu\textsuperscript{\rm 1},
}
\affiliations {
    \textsuperscript{\rm 1} School of Electronic Engineering, Key Laboratory of Collaborative Intelligence Systems of Ministry of Education, \\Xidian University, Xi'an, China\\
    \textsuperscript{\rm 2} School of Computer Science and Technology, Key Laboratory of Collaborative Intelligence Systems of Ministry of Education, Xidian University, Xi'an, China \\
    \textsuperscript{\rm 3} Academy of Advanced Interdisciplinary Research, Key Laboratory of Collaborative Intelligence Systems of Ministry of Education, Xidian University, Xi'an, China \\
    xiaolongfan@outlook.com, gong@ieee.org, \{ywu, zdtang, jieyiliu\}@xidian.edu.cn
}

\begin{document}

\maketitle

\begin{abstract}
Graph Structure Learning (GSL) has demonstrated considerable potential in the analysis of graph-unknown non-Euclidean data across a wide range of domains. However, constructing an end-to-end graph structure learning model poses a challenge due to the impediment of gradient flow caused by the nearest neighbor sampling strategy. In this paper, we construct a differential graph structure learning model by replacing the non-differentiable nearest neighbor sampling with a differentiable sampling using the reparameterization trick. Under this framework, we argue that the act of sampling \mbox{nearest} neighbors may not invariably be essential, particularly in instances where node features exhibit a significant degree of similarity. To alleviate this issue, the bell-shaped Gaussian Similarity (GauSim) modeling is proposed to sample non-nearest neighbors. To adaptively model the similarity, we further propose Neural Gaussian Similarity (NeuralGauSim) with learnable parameters featuring flexible sampling behaviors. In addition, we develop a scalable method by transferring the large-scale graph to the transition graph to significantly reduce the complexity. Experimental results demonstrate the effectiveness of the proposed methods.   
      
\end{abstract}

\section{Introduction}
In recent years, there has been a notable surge in academic attention towards Graph Neural Networks (GNNs) \cite{gcn, graphsaint, linkx, fan2023maximizing}. A plethora of graph neural network \mbox{models} have been introduced, showcasing noteworthy advancements across diverse domains such as social network analysis \cite{gao2023survey}, natural language processing \cite{meng2022gnnlm}, computer vision \cite{han2022vision}, and various other fields. The efficacy of graph neural network can be attributed to their inherent capability of effectively leveraging the abundant information present within both the structure of graph topology and the input node attributes in a concurrent manner. However, the graph structure is not invariably ascertainable. For instance, within a social network, the interconnections among users encompass sensitive privacy aspects, impeding direct access to such information.

To alleviate this issue, Graph Structure Learning (GSL) \cite{idgl,slaps,nodeformer} is proposed to jointly learn the latent graph structure and corresponding graph embeddings using structure learner and graph neural network encoder where the parameters are optimized together by the downstream task. Specifically, the structure learner first computes the similarity between node feature pairs as the edge \mbox{weights} by using similarity kernel function, such as inner-product kernel \cite{yu2021graph,zhao2021data}, cosine similarity kernel \cite{amgcn,idgl}, and diffusion kernel \cite{gasteiger2019diffusion}. Then the graph structure is generated via a structure sampling process from edge weight distribution. After producing the graph structure, graph neural network encoder takes the node features and the generated graph structure as input to produce the final node embeddings for downstream tasks. However, two fundamental weakness of this framework may limit the performance and scalability of graph structure learning method. First, \textit{the generated graph structure is not differentiable with respect to the edge weight distribution due to discrete sampling blocking gradient flow.} Second, \textit{computing the edge similarity for all pairs of graph nodes requires huge complexity for both computational time and memory consumption, rendering significant scalability issue for large graphs.}

To solve the problem of non-differentiable sampling, we utilize the concrete relaxation of the Categorical distribution by replacing the non-differentiable sampling with a differentiable sampling mechanism using the Gumbel-Softmax distribution. The Gumbel-Softmax distribution \cite{gumbel,neuralspar,gsl-vib} is a reparameterization trick that allows for the generation of discrete samples while maintaining differentiability. Note that the sampling probability is positively correlated with edge similarity, whereby greater similarity values entail higher probabilities of sampling the edges. We term this sampling approach as the linear sampling strategy. Our analysis reveals that this linear sampling method is not universally indispensable, particularly in situations where node features demonstrate a substantial level of similarity. To alleviate this issue, we propose a bell-shaped Gaussian Similarity (GauSim) modeling strategy, enabling the edge sampling probability entails an initial increase followed by a decrease as the similarity between node pairs diminishes. Note that \mbox{different} Gaussian function parameters need to be set for different \mbox{node} pairs, we further propose a Neural Gaussian Similarity (NeuralGauSim) modeling \mbox{strategy} to adaptively learn the bell-shaped Gaussian function parameters.

To address the scalability issue, we develop a transition graph structure learning method by involving the transformation of the initial node set into the more streamlined transition node set. Specifically, we first project the initial feature matrix into the transition feature matrix. Then for every node in the original graph, we calculate a similarity score between the corresponding node in the original graph and its counterpart in the transferred graph. By leveraging this similarity score, we differentiably sample the edges to generate the desired graph structure. Finally, the node embeddings can be generated by using the graph neural network encoder which takes the generated graph structure and the node features of the transition graph as input. Different from the previous anchor-based graph structure leaning method \cite{idgl}, which randomly samples a set of anchors from the initial graph, the developed transition graph method adopts projection matrices to learn to generate the transition graph, thus mitigating the information loss that typically arises from random sampling.          

Extensive experiments on graph and graph-enhanced application datasets demonstrate the superior effectiveness of the proposed method. To summarize, we outline the main contributions in this paper as follows: 
\begin{enumerate}
	\item We propose the neural Gaussian similarity modeling for differential graph structure learning to alleviate the issue of structure sampling. 
	\item We develop the transition graph structure learning by transferring the initial graph to the transition graph to reduce the complexity.  
	\item Extensive experiments on graph and graph-enhanced application datasets demonstrate the superior effectiveness of the proposed method.
\end{enumerate}

\section{Related Work}
\subsection{Graph Neural Network}
Graph Neural Networks (GNNs) aim to model the non-Euclidean data structure and have been demonstrated to achieve state-of-the-art performance on graph analysis tasks. As a unified framework for
graph neural networks, Message Passing Neural Network (MPNN) \cite{scarselli2008graph,9450014,9508847} generalizes the several existing representative graph neural networks, such as GCN \cite{gcn}, GAT \cite{gat}, GraphSAINT \cite{graphsaint}, AM-GCN\cite{amgcn}, and LINKX \cite{linkx}, which consists of two functions, i.e., message passing function and readout function. Most empirical studies of graph neural networks directly take the observed graph as input. However, the graph structure is not invariably ascertainable in practice. In this paper, we focus on the graph-unknown non-Euclidean data representation learning.

\subsection{Graph Structure Learning}
Graph Structure Learning (GSL) targets at jointly learning an optimized graph structure and its corresponding representations. A typical GSL model involves two trainable components, i.e., structure learner and graph neural network encoder. The structure learner is an encoding
function that models the optimal graph structure represented in edge weights. In recent years, several structure learners have been proposed, such as LDS \cite{lds}, Pro-GNN \cite{prognn}, IDGL \cite{idgl}, SLAPS \cite{slaps}, and NodeFormer \cite{nodeformer}, and achieved significant performance improvement. In this paper, we propose the Gaussian similarity modeling strategy and transition graph structure learning method to alleviate the issues of structure sampling and scalability. 

\section{Exploring Graph Structure Learning} 

\subsection{Problem Definition}
Let $\mathcal{G} = (\mathcal{V}, \mathcal{E})$ be a graph with $\mathcal{V}$ and $\mathcal{E}$ denoting the \mbox{node} set and edge set, respectively. The node feature matrix is denoted by $X = \{x_1,..,x_i,...x_n\}$ where $x_i \in \mathbb{R}^{d}$ is the attribute of node $v_i$. The graph structure is described by the adjacency matrix $A \in \{0, 1\}^{n\times n}$ for graphs where $A_{ij} = 1$ indicates $(v_i, v_j) \in \mathcal{E}$. In general, a GNN encoder, parameterized by $\Theta$, receives the graph structure and node features as input, then produces node embeddings $H \in \mathbb{R}^{n\times m}$ for downstream tasks. This paper primarily centers on the exploration of graph representation learning in the context of the unknown graph structure. In this setting, graph structure learning can be formulated as producing the graph structure $A^*$ and its corresponding node embeddings $H = \textrm{GNN}(A^*, X)$ with respect to the downstream tasks.

\begin{figure*}[t]
	\centering
	\includegraphics[width=2.1\columnwidth]{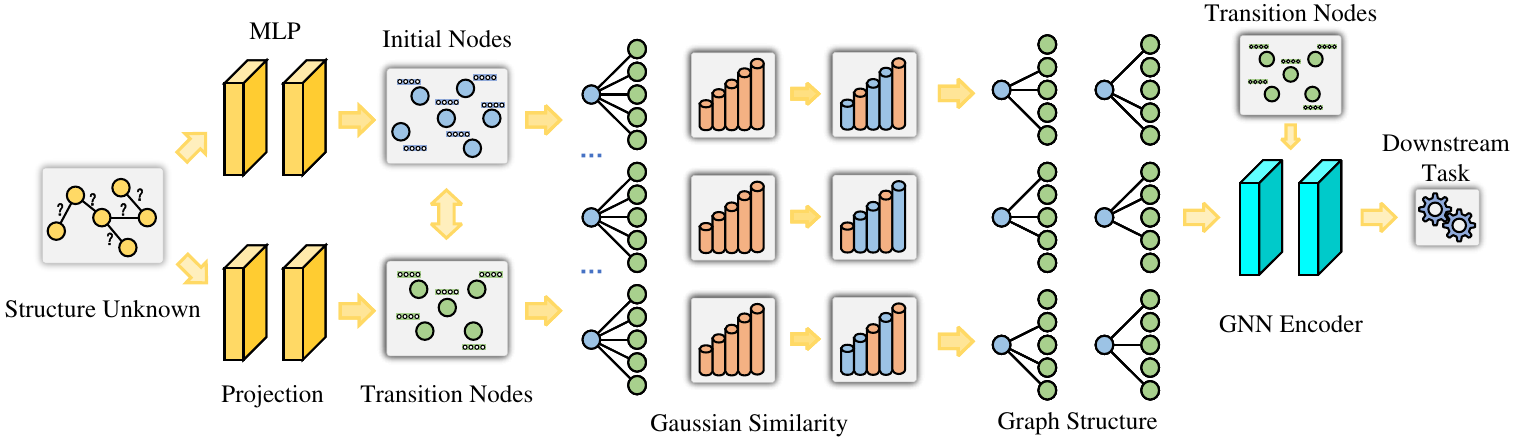}
	\caption{The proposed neural Gaussian similarity modeling and transition graph structure learning strategy. }
	\label{framework}
	\vspace{-1em}
\end{figure*}

\subsection{Differential Graph Structure Learning}
Graph structure learning aims to jointly learn the graph structure and corresponding graph embeddings. Given node features, graph structure learning first models the similarity between node feature pairs in form of 
\begin{gather}
	z_i = \textrm{MLP}(x_i) \label{eq1}\\
	\pi_{i,j} = \frac{\exp(z_iz_j^{\top})}{\sum_{w=1}^{n}\exp(z_iz_w^{\top})} \label{eq2}
\end{gather}
where $\textrm{MLP}(\cdot)$ denotes a multi-layer perceptron and $\pi_{i,j}$ denotes the edge similarity between node $v_i$ and $v_j$.
Then the graph structure $A^*$ can be generated using a structure sampling strategy that models all potential edges as a collection of independent Categorical random variables. These variables are parameterized by the learned similarity $\pi$ in form of
\begin{gather}
	A^* = \bigcup_{v_i,v_j \in \mathcal{V}}\left\{A_{i,j} \sim \textrm{Cat}(\pi_{i,j})\right\} \label{eq3}
\end{gather}  
where $A_{i,j} \sim \textrm{Cat}(\pi_{i,j})$ denotes the edge sampling process from Categorical distribution. Here, the similarity $\pi_{i,j}$ describes the edge sampling probability and smaller $\pi_{i,j}$ indicates that the edge $(v_i, v_j)$ tends to be removed. In general, we sample $\mathcal{K}$ times for each node $v_i$ to form neighbors. However, this method poses a challenge that the graph structure $A^*$ is not differentiable with respect to $\pi$ due to discrete sampling blocking gradient flow. To solve this problem, we can utilize the concrete relaxation of the Categorical distribution \cite{gumbel,neuralspar,gsl-vib} by replacing the non-differentiable sample from the Categorical distribution with a differentiable sample from the Gumbel-Softmax distribution, i.e.,
\begin{gather}
	\text{Cat}(\pi_{i,j}) \approx \frac{\exp((\log (\pi_{i,j}) + g_i)/\tau)}{\sum_{w=1}^{n}\exp((\log (\pi_{iw}) + g_w)/\tau)}
\end{gather}   
where $g_i = -\log(-\log(\epsilon))$ with $\epsilon$ randomly drawn from $\text{Uniform}(0,1)$ and $\tau \in \mathbb{R}^+$ is the temperature which controls the interpolation between the discrete distribution and continuous categorical densities. After obtaining the graph structure $A^*$, we can use a GNN encoder $\textrm{GNN}(A^*, X)$, e.g., GCN \cite{gcn}, to produce the final node embeddings for downstream tasks.

\subsection{Analysis of Structure Sampling}
Note that the sampling probability and similarity of the edge $(v_i, v_j)$ exhibit a linear relationship, wherein the greater the similarity value, the higher the probability of sampling that edge. In this paper, we aim to provide a better understanding of this sampling strategy by asking the following question: \textit{is this linear sampling strategy always necessary?} To answer this question, without loss of generality, we use GCN as GNN encoder and conduct the following theorem.  

\begin{theorem}
	Suppose the number of structure sampling be $\mathcal{K}$ for each node, $h_i$ be the feature embedding of node $v_i$, $h_j$ be the sampled neighbors of node $v_i$ where $j\in[1,\ldots,\mathcal{K}]$, and node features are normalized using 2-norm normalization. If $\lVert h_i - h_j\rVert_2 \leq \varepsilon$ for $\forall j\in [1,\ldots,\mathcal{K}]$ where $\varepsilon$ is a non-negative constant, then
	\begin{gather}
		\lVert \hat{h}_i - h_i \rVert_2 \leq \varepsilon
	\end{gather}  
	where $\hat{h}_i = \sum_{j=1}^{\mathcal{K}}\frac{1}{\sqrt{d_id_j}}h_j$ is the predict probability distribution after graph convolutional operator.
	\label{the1}
\end{theorem}

\begin{proof}
	Given $\lVert h_i - h_j\rVert_2 \leq \varepsilon$ for $\forall j\in [1,\ldots,\mathcal{K}]$ where $\varepsilon$ is a non-negative constant, then we have
	\begin{gather}
		\lVert h_i - h_j\rVert_2^2 \leq \varepsilon^2\notag\\
		\lVert h_i \rVert_2^2 + \lVert h_j \rVert_2^2 - 2\cdot\langle h_i, h_j \rangle \leq \varepsilon^2\\
		\langle h_i, h_j \rangle \geq \frac{\lVert h_i \rVert_2^2 +\lVert h_j \rVert_2^2 -\varepsilon^2}{2}.\notag
	\end{gather}
	After graph convolutional operator, we can denote the 2-norm difference as
	\begin{gather}
		\lVert \hat{h}_i - h_i \rVert_2
		=\sqrt{ \lVert \hat{h}_i \rVert_2^2 + \lVert h_i \rVert_2^2 - 2\cdot\langle h_i, \hat{h}_i \rangle }.
	\end{gather}
	For $\langle \hat{h}_i, h_i \rangle$, we can get 
	\begin{gather}
		\begin{aligned}
		\langle h_i, \hat{h}_i \rangle &= \langle h_i,  \sum_{j=1}^{\mathcal{K}}\frac{1}{\sqrt{d_id_j}}h_j \rangle\\
		&= \frac{1}{\mathcal{K}} \sum_{j=1}^\mathcal{K} \langle h_i, h_j \rangle \geq \frac{1}{\mathcal{K}} \sum_{j=1}^\mathcal{K} \frac{2-\varepsilon^2}{2}.
		\end{aligned}
	\end{gather}
	Therefore, the 2-norm difference can be represented as
	\begin{gather}
	\begin{aligned}
		\lVert \hat{h}_i &- h_i \rVert_2
		=\sqrt{ \lVert \hat{h}_i \rVert_2^2 + \lVert h_i \rVert_2^2 - 2\cdot\langle h_i, \hat{h}_i \rangle } \\
		&\leq \sqrt{\lVert \hat{h}_i \rVert_2^2 + \lVert h_i \rVert_2^2 - \frac{2}{\mathcal{K}} \sum_{j=1}^\mathcal{K} \frac{2-\varepsilon^2}{2}} \\
		&\leq \sqrt{2-(2-\varepsilon^2)} = \varepsilon.
	\end{aligned}
	\end{gather}
	\vspace{-1.1em}
\end{proof}
Theorem \ref{the1} tells us that the act of linear sampling may not invariably be essential, particularly in instances where node features exhibit a significant degree of similarity. In this case, the integration of the learned graph structure does not yield notable increase in informational gain. 

\section{Proposed Method}
In this section, we present the proposed Neural Gaussian Similarity Modeling to alleviate the issue of structure sampling and the Transition Graph Structure Learning to reduce the complexity of computing edge similarity. The overall framework is shown in Figure \ref{framework}. 

\subsection{Neural Gaussian Similarity Modeling}
From the analysis of structure sampling, we find that it is not invariably essential to exclusively sample the edges exhibiting the highest probability distribution, i.e., the nodes displaying the utmost similarity. To alleviate the issue, a nature idea is obtaining the Difference Similarity (DiffSim) between pair of nodes in form of 
\begin{gather}
	\pi_{i,j} = \frac{\exp(z_i(z_i-z_j)^{\top})}{\sum_{w=1}^{n}\exp(z_i(z_i-z_w)^{\top})}
\end{gather}
where $z_i(z_i-z_j)^{\top} = 1 - z_iz_j^{\top}$ when $z_i, z_j$ are normalized using 2-norm normalization. As shown in Figure \ref{graph0}, we can observe that the sampling probability of edges increases as the similarity decreases, resulting in highly dissimilar node pairs being sampled as neighbors. This difference similarity modeling does not satisfy the assumption of the homophily effect in networks \cite{xie2020gnns,ma2022is} and may harm the performance of GNN model. 

To fulfill the requirement of the edge sampling probability, which entails an initial increase followed by a decrease as the similarity between node pairs diminishes, we propose a novel Gaussian Similarity (GauSim) modeling strategy in form of 
\begin{gather}
	\phi(z_i, z_j) = \exp\left(-\frac{(z_iz_j^{\top} - b)^2}{c}\right) \label{eq11}\\
	\pi_{i,j} = \frac{\exp(\phi(z_i, z_j))}{\sum_{w=1}^{n}\exp(\phi(z_i, z_w))}
\end{gather}
where $b, c$ are the parameters of Gaussian function. By choosing appropriate $b$ and $c$, Gaussian similarity can be used as edge sampling probability to meet the nonlinear sampling requirement. Here, we set $b=0.5$ and $c=0.02e$, and the relationship between sampling probability and similarity is shown in the Figure \ref{graph0}. 

\begin{figure}[t]
	\centering
	\subfigure[LinSim]{\includegraphics[width=0.326\columnwidth]{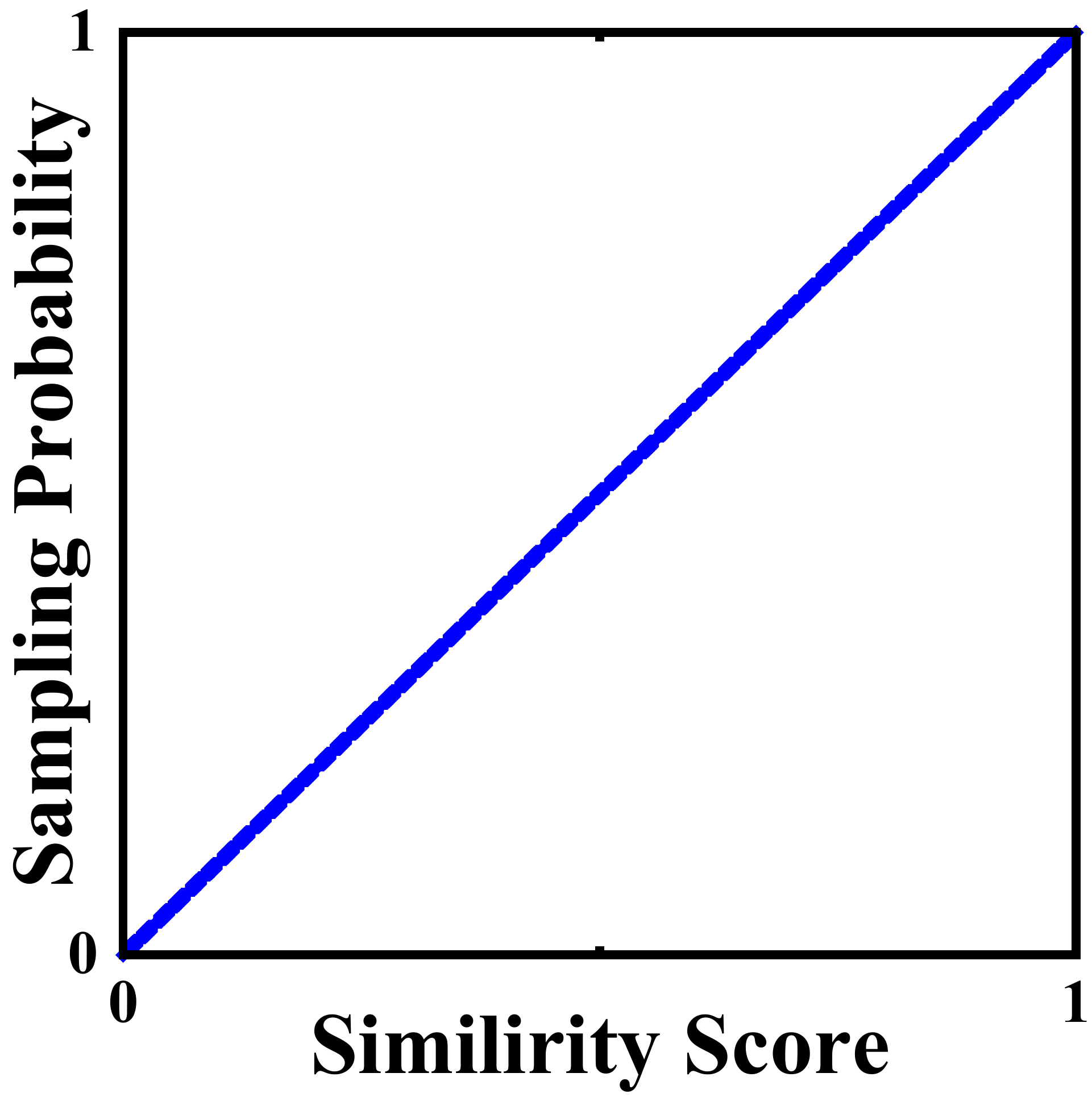}} 
	\subfigure[DiffSim]{\includegraphics[width=0.326\columnwidth]{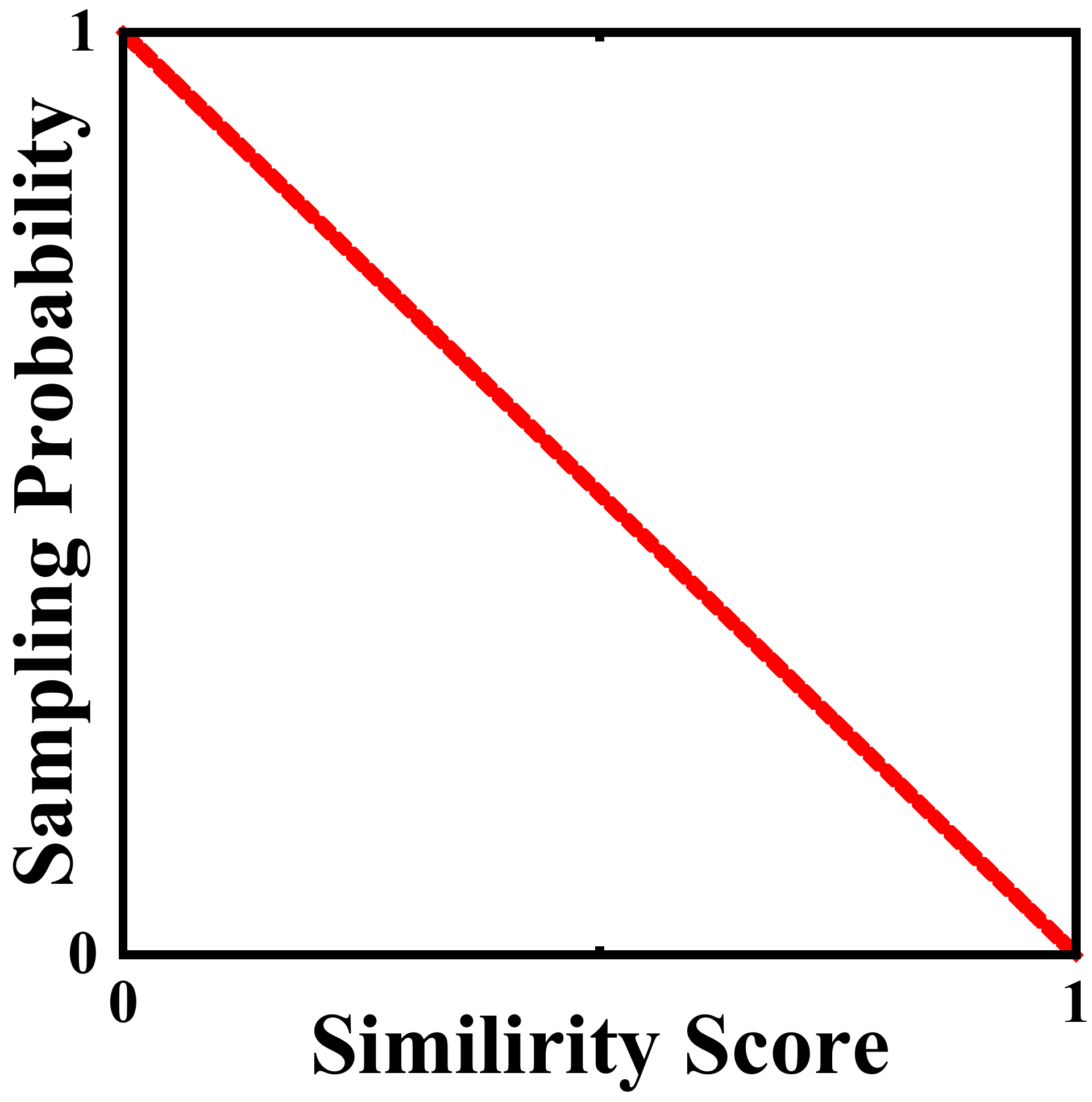}} 
	\subfigure[GauSim]{\includegraphics[width=0.326\columnwidth]{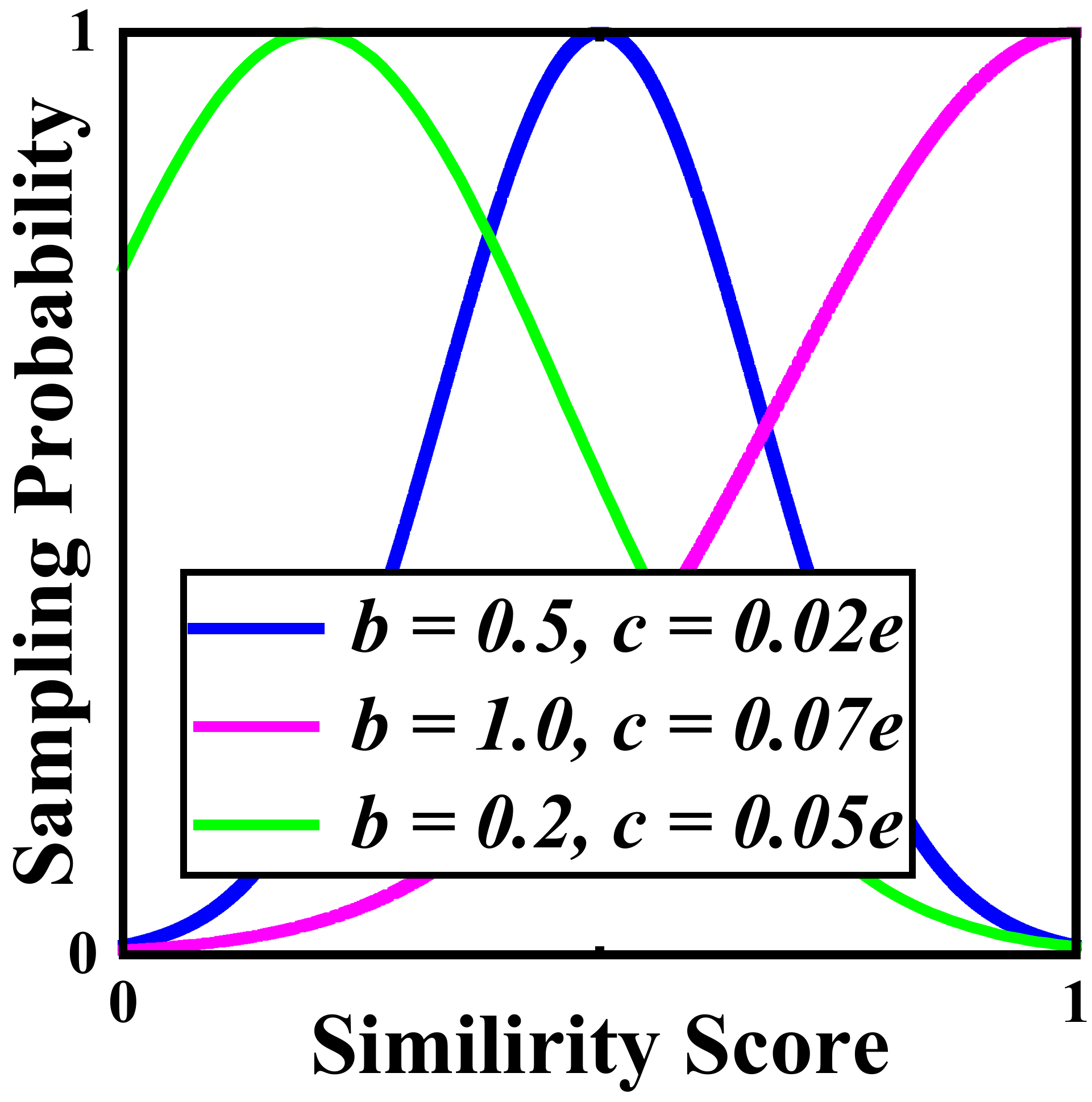}} 
	\caption{Relationship between sampling probability and similarity score. Subfigure (a) denotes the Linear Sampling (LinSim) strategy, subfigure (b) denotes the Difference Similarity (DiffSim), and subfigure (c) denotes the Gaussian Similarity (GauSim) with different parameters. }
	\label{graph0}
	\vspace{-1em}
\end{figure}

Note that different node pairs may need to set different Gaussian function parameters, we further propose a Neural Gaussian Similarity (NeuralGauSim) modeling to adaptively learn the Gaussian function parameters. Specifically, we define a parameter learning network in form of 
\begin{gather}
	b_i = \sigma(z_iw_b^{\top}), c_i = \sigma(z_iw_c^{\top})\\
	\phi(z_i, z_j) = \exp\left(-\frac{(z_iz_j^{\top} - b_i)^2}{c_i}\right) \label{eq14}\\
	\pi_{i,j} = \frac{\exp(\phi(z_i, z_j))}{\sum_{w=1}^{n}\exp(\phi(z_i, z_w))}
\end{gather}
where $w_b, w_c \in \mathbb{R}^{m}$ are the learnable parameters and $\sigma$ is Sigmoid activation function. The relationship between sampling probability and similarity with different parameters $b$ and $c$ is shown in the Figure \ref{graph0}. From this figure, we can observe that the neural Gaussian similarity modeling can generalize the sampling strategy based on the previous linear and difference similarity by setting appropriate parameters. 

\subsection{Transition Graph Structure Learning}
\begin{table*}[t]
	\centering
	\setlength{\tabcolsep}{5mm}
	\scalebox{0.67}{
		\begin{tabular}{|c|c|c|c|c|c|c|c|c|}
			\toprule
			\textsc{Method}&CiteSeer&PubMed&Chameleon&Squirrel&CS&Physics&20News&Mini-ImageNet\\
			\midrule
			\midrule
			$\text{GCN}_{knn}$&66.67$\pm$1.33&78.84$\pm$0.35&41.91$\pm$1.13&26.80$\pm$0.82&87.57$\pm$0.11&91.84$\pm$0.21&61.36$\pm$0.13&80.69$\pm$0.48\\
			$\text{HeatGCN}_{knn}$&65.07$\pm$0.52&75.20$\pm$0.28&43.23$\pm$0.20&27.21$\pm$0.55&88.21$\pm$0.13&91.96$\pm$0.13&50.17$\pm$0.61&78.14$\pm$0.17\\
			$\text{LINKX}_{knn}$&57.80$\pm$1.70&75.72$\pm$1.13&36.32$\pm$2.47&24.70$\pm$2.42&80.38$\pm$2.62&81.28$\pm$2.24&52.24$\pm$3.14&71.19$\pm$2.50\\
			SLAPS&68.23$\pm$0.22&79.55$\pm$0.23&43.44$\pm$0.36&27.28$\pm$0.96&85.74$\pm$2.12&92.14$\pm$2.13&60.97$\pm$0.20&80.23$\pm$1.14\\
			IDGL-Full&68.67$\pm$0.32&OOM&45.47$\pm$0.76&28.23$\pm$0.19&OOM&OOM&62.53$\pm$0.24&OOM\\
			IDGL-Anchor&67.31$\pm$0.07&79.69$\pm$0.43&45.35$\pm$1.23&28.08$\pm$1.17&89.15$\pm$0.81&OOM&61.91$\pm$0.26&80.51$\pm$0.61\\
			NodeFormer&67.19$\pm$0.28&82.57$\pm$0.45&47.02$\pm$0.61&27.80$\pm$0.75&\textbf{91.77$\pm$0.10}&\textbf{94.81$\pm$0.11}&62.70$\pm$0.84&83.10$\pm$0.49\\
			\midrule
			LinSim&67.23$\pm$0.23&OOM&44.91$\pm$1.10&27.26$\pm$0.95&OOM&OOM&54.40$\pm$1.07&OOM\\
			LinSim-T&69.51$\pm$0.60&81.41$\pm$1.31&43.63$\pm$1.38&27.75$\pm$0.54&87.67$\pm$1.55&93.88$\pm$0.19&60.84$\pm$0.26&81.94$\pm$0.14\\
			\midrule
			GauSim&69.19$\pm$0.14&OOM&47.84$\pm$0.54&28.13$\pm$0.87&OOM&OOM&62.49$\pm$0.20&OOM\\
			GauSim-T&\textbf{70.19$\pm$0.84}&\underline{82.62$\pm$0.51}&47.19$\pm$1.16&27.80$\pm$0.67&88.94$\pm$0.47&94.13$\pm$0.32&61.35$\pm$0.56&\underline{83.33$\pm$0.88}\\
			\midrule
			NeuralGauSim&69.31$\pm$0.35&OOM&\underline{48.02$\pm$0.20}&\textbf{28.88$\pm$1.08}&OOM&OOM&\underline{62.85$\pm$0.17}&OOM\\
			NeuralGauSim-T&\underline{70.15$\pm$0.37}&\textbf{82.65$\pm$0.74}&\textbf{48.25$\pm$0.63}&\underline{28.57$\pm$0.36}&\underline{89.22$\pm$1.55}&\underline{94.64$\pm$0.10}&\textbf{62.89$\pm$1.10}&\textbf{84.89$\pm$0.37}\\
			\bottomrule
	\end{tabular}}
	\caption{Experimental results comparison with baselines. Here, OOM denotes Out-of-Memory.}
	\label{tab1}
\end{table*}

\begin{table*}[t]
	\centering
	\setlength{\tabcolsep}{13mm}
	\scalebox{0.67}{
		\begin{tabular}{|c|c|c|c|c|c|}
			\toprule
			&\textsc{Method}&CiteSeer&Chameleon&Squirrel&20News\\
			\midrule
			\midrule
			\multirow{3}*{$\mathcal{K}$=5}
			&LinSim&67.23$\pm$0.23&44.91$\pm$1.10&27.26$\pm$0.95&54.40$\pm$1.07\\
			&GauSim&69.19$\pm$0.14&47.84$\pm$0.54&28.13$\pm$0.87&62.49$\pm$0.20\\
			&NeuralGauSim&\textbf{69.31$\pm$0.35}&\textbf{48.02$\pm$0.20}&\textbf{28.88$\pm$1.08}&\textbf{62.85$\pm$0.17}\\
			\midrule
			\multirow{3}*{$\mathcal{K}$=10}
			&LinSim&66.10$\pm$0.82&43.63$\pm$0.82&26.59$\pm$0.87&\multirow{3}*{OOM}\\
			&GauSim&68.31$\pm$0.52&46.73$\pm$0.54&27.29$\pm$0.67&\\
			&NeuralGauSim&\textbf{68.35$\pm$0.98}&\textbf{46.84$\pm$0.46}&\textbf{27.47$\pm$0.62}&\\
			\midrule
			\multirow{3}*{$\mathcal{K}$=15}
			&LinSim&65.55$\pm$0.82&42.79$\pm$1.95&25.34$\pm$1.14&\multirow{3}*{OOM}\\
			&GauSim&69.03$\pm$0.21&47.54$\pm$0.70&28.03$\pm$1.05&\\
			&NeuralGauSim&\textbf{69.55$\pm$0.96}&\textbf{47.49$\pm$1.67}&\textbf{28.49$\pm$0.35}&\\
			\midrule
			\multirow{3}*{$\mathcal{K}$=20}
			&LinSim&64.39$\pm$0.76&41.73$\pm$1.09&25.13$\pm$1.14&\multirow{3}*{OOM}\\
			&GauSim&68.47$\pm$0.97&47.66$\pm$0.44&26.77$\pm$1.19&\\
			&NeuralGauSim&\textbf{68.69$\pm$0.52}&\textbf{48.84$\pm$0.98}&\textbf{27.95$\pm$0.77}&\\
			\bottomrule
	\end{tabular}}
	\caption{Experimental results comparison with different samples $\mathcal{K}$ for full-graph models. Here, OOM denotes Out-of-Memory.}
	\label{tab2}
	\vspace{-1em}
\end{table*}

The edge similarity computes similarity scores for all pairs of graph nodes, which requires $\mathcal{O}(n^2)$ complexity for both computational time and memory consumption, rendering significant scalability issue for large graphs. To address the scalability issue, we develop a transition graph similarity modeling method by transferring the initial graph comprising $n$ nodes to the more streamlined transition graph containing $s$ nodes. Specifically, we first project the initial nodes $X \in \mathbb{R}^{n\times d}$ into the transition nodes in form of 
\begin{gather}
	R = W_tX \in \mathbb{R}^{s\times d}
\end{gather} 
where $W_t \in \mathbb{R}^{s\times n}$ is the projection matrix and $R = \{r_1,...,r_s\}$ is transition node features. Then for each nodes of the initial graph, the similarity scores are calculated by 
\begin{gather}
	z_i = \textrm{MLP}(x_i) \\
	\pi_{i,j} = \frac{\exp(\phi(z_i, r_j))}{\sum_{w=1}^{s}\exp(\phi(z_i, r_w))}
\end{gather}
where $\phi(z_i, r_i)$ can use the Eqn. \eqref{eq11} and Eqn. \eqref{eq14} as the similarity measurement. Note that for node $v_i$, the calculation range of edge similarity score is the transition graph with $s$ nodes instead of the original graph with $n$ nodes, which requires $\mathcal{O}(ns)$ complexity for both computational time and memory consumption. The parameter $s$ is often set as $s\ll n$ in practice, hence the complexity of computing similarity scores requires $\mathcal{O}(n)$ for both computational time and memory consumption.
After obtaining the edge similarity scores, we can utilize the concrete relaxation of Categorical distribution to differentiably sample the edges by  
\begin{gather}
	A^*_t = \bigcup_{v_i \in \mathcal{V}, v_j \in \mathcal{V}_t}\left\{A_{i,j} \sim \textrm{Cat}(\pi_{i,j})\right\} \in \mathbb{R}^{n\times s}
\end{gather}  
where $\mathcal{V}_t$ denotes the node set of the transition graph. Finally, we can utilize the GCN encoder $\textrm{GNN}(\cdot)$ to produce the node embeddings in form of 
\begin{gather}
	X_t = W_eX \in \mathbb{R}^{s\times d}\\
	H = \textrm{GNN}(A^*_t, X_t) \in \mathbb{R}^{n\times m}
\end{gather} 
where $W_e \in \mathbb{R}^{s\times n}$ is the learnable projection matrix and $H$ is the final node embedding matrix. Different from the previous anchor-based graph structure leaning method \cite{idgl}, which randomly samples a set of anchors from the initial graph, the developed transition graph method uses projection matrices to learn the node transition strategy, thus mitigating the information loss resulting from random sampling.  

\section{Experiment}

\begin{table*}[t]
	\centering
	\setlength{\tabcolsep}{4mm}
	\scalebox{0.7}{
		\begin{tabular}{|c|c|c|c|c|c|c|c|c|c|}
			\toprule
			&\textsc{Method}&CiteSeer&PubMed&Chameleon&Squirrel&CS&Physics&20News&Mini-ImageNet\\
			\midrule
			\midrule
			\multirow{3}*{\rotatebox{90}{$\mathcal{K}$=5}}
			&LinSim-T&69.51$\pm$0.60&81.41$\pm$1.31&43.63$\pm$1.38&27.75$\pm$0.54&87.67$\pm$1.55&93.88$\pm$0.19&60.84$\pm$0.26&81.94$\pm$0.14\\
			&GauSim-T&\textbf{70.19$\pm$0.84}&82.62$\pm$0.51&47.19$\pm$1.16&27.80$\pm$0.67&88.94$\pm$0.47&94.14$\pm$0.32&61.35$\pm$0.56&83.33$\pm$0.88\\
			&NeuralGauSim-T&70.15$\pm$0.37&\textbf{82.65$\pm$0.74}&\textbf{48.25$\pm$0.63}&\textbf{28.57$\pm$0.36}&\textbf{89.22$\pm$1.55}&\textbf{94.64$\pm$0.10}&\textbf{62.89$\pm$1.10}&\textbf{84.89$\pm$0.37}\\
			\midrule
			\multirow{3}*{\rotatebox{90}{$\mathcal{K}$=10}}
			&LinSim-T&64.23$\pm$0.70&79.64$\pm$1.76&42.40$\pm$1.06&25.18$\pm$1.50&83.49$\pm$1.18&92.30$\pm$1.36&59.60$\pm$0.77&80.65$\pm$0.98\\
			&GauSim-T&69.03$\pm$0.68&80.42$\pm$1.19&47.49$\pm$1.07&27.08$\pm$1.12&88.63$\pm$0.69&\textbf{94.12$\pm$0.76}&\textbf{60.36$\pm$0.12}&82.17$\pm$1.24\\
			&NeuralGauSim-T&\textbf{69.19$\pm$0.76}&\textbf{80.69$\pm$1.62}&\textbf{47.66$\pm$0.20}&\textbf{27.57$\pm$0.35}&\textbf{89.79$\pm$1.17}&94.06$\pm$0.60&60.06$\pm$0.57&\textbf{83.53$\pm$0.53}\\
			\midrule
			\multirow{3}*{\rotatebox{90}{$\mathcal{K}$=15}}
			&LinSim-T&64.75$\pm$0.92&76.46$\pm$1.12&40.94$\pm$1.16&24.42$\pm$1.98&79.48$\pm$1.54&88.02$\pm$1.86&58.87$\pm$0.61&80.55$\pm$0.74\\
			&GauSim-T&70.03$\pm$0.57&81.61$\pm$0.43&47.60$\pm$0.44&28.16$\pm$0.38&89.51$\pm$0.34&92.17$\pm$1.87&60.19$\pm$1.07&82.64$\pm$0.71\\
			&NeuralGauSim-T&\textbf{70.39$\pm$0.80}&\textbf{81.70$\pm$0.55}&\textbf{47.31$\pm$0.51}&\textbf{28.49$\pm$0.42}&\textbf{90.49$\pm$0.49}&\textbf{92.89$\pm$1.67}&\textbf{60.36$\pm$0.97}&\textbf{82.87$\pm$0.97}\\
			\midrule
			\multirow{3}*{\rotatebox{90}{$\mathcal{K}$=20}}
			&LinSim-T&64.31$\pm$0.53&75.03$\pm$1.57&41.35$\pm$1.62&24.62$\pm$1.13&71.68$\pm$2.92&87.91$\pm$1.39&57.54$\pm$1.11&79.70$\pm$1.14\\
			&GauSim-T&\textbf{68.99$\pm$0.92}&80.73$\pm$1.35&47.99$\pm$0.88&28.29$\pm$0.35&87.87$\pm$1.17&92.30$\pm$0.76&59.04$\pm$0.44&80.24$\pm$0.70\\
			&NeuralGauSim-T&68.59$\pm$0.98&\textbf{81.50$\pm$0.87}&\textbf{48.04$\pm$0.46}&\textbf{28.54$\pm$0.38}&\textbf{88.54$\pm$1.12}&\textbf{93.14$\pm$0.59}&\textbf{59.59$\pm$0.76}&\textbf{81.64$\pm$0.97}\\
			\bottomrule
	\end{tabular}}
	\caption{Experimental results comparison with different samples $\mathcal{K}$ for transition-graph models.}
	\label{tab3}
	\vspace{-0.5em}
\end{table*}

\subsection{Datasets}

To evaluate the developed methods, we use eight commonly used graph datasets, including two citation network datasets, i.e., CiteSeer and PubMed \cite{yang2016revisiting}, two Wikipedia network datasets, i.e., Chameleon and squirrel \cite{geom-gcn}, two coauthor network datasets, i.e., CS and Physics \cite{shchur2018pitfalls}, and two graph-enhanced application datasets, i.e., 20News and Mini-ImageNet \cite{nodeformer}. 
%More details of datasets can be found in Appendix.   

\subsection{Baselines}
We compare the proposed method with prominent existing graph structure learning baselines, including $\text{GCN}_{knn}$ \cite{gcn}, $\text{HeatGCN}_{knn}$ where the heat \mbox{kernel} \cite{amgcn} serves as similarity \mbox{measurement}, $\text{LINKX}_{knn}$ \cite{linkx}, SLAPS \cite{slaps}, IDGL-Full, IDGL-Anchor \cite{idgl}, and NodeFormer \cite{nodeformer}. For kNN based methods, we create a kNN graph based on the node feature similarities and feed this graph to the GNN model. 

\subsection{Configurations}
We implement the proposed model as well as the baselines with the well-known Pytorch Geometric library and all experiments conduct on a NVIDIA GeForce 3090 card with 24GB memory.
The experiments are repeated five times for all algorithms to finally report the average accuracy and the corresponding standard deviation. 
For all datasets, we randomly sampled $50\%$ of nodes for training set, $25\%$ for validation set, and $25\%$ for test set.
The Adam optimizer are used with the learning rate $0.01$ for CiteSeer dataset and $0.001$ for other datasets. We set the hidden dimension to be $256$ for Mini-ImageNet dataset, $32$ for other datasets, and the number of transition-graph nodes to $500$. The weight parameters are initialized using Glorot initialization and the bias parameters using zero initialization. For parameters $c$ of NeuralGauSim, we multiply $c$ by a scale factor of $0.1$.
Following NodeFormer \cite{nodeformer}, the temperature $\tau$ is set to be $0.25$. We add dropout layers with dropout probabilities of $0.5$ after the first layer of the GNNs, and train two-layer GNNs. In addition, we set all experiments with the same seed to achieve a fair comparison.

\subsection{Results Comparison with Baselines}
To verify the performance of the developed models in comparison to baselines, we conduct comparative experiments and report the experimental results in Table \ref{tab1}. In this experiment, we set the parameter K of kNN to be $5$ for kNN based baselines, and the neighbor mask parameter to be $0$ for IDGL-Full and IDGL-Anchor. From this table, we have the following observations. First, compared with the previous graph structure learning baselines, the proposed Gaussian similarity modeling methods achieve better performance on six datasets out of eight, indicating the effectiveness of the differential graph structure learning. Second, compared with the linear similarity, the proposed Gaussian similarity modeling methods consistently achieve better performance, demonstrating the effectiveness of the proposed methods. Third, comparing with transition-graph based models, i.e., LinSim-T, GauSim-T, and NeuralGauSim-T, and full-graph models, i.e. LinSim, GauSim, and NeuralGauSim, we can observe that transition-graph based models achieves performance improvements while reducing complexity. This implies the effectiveness of introducing the learnable transition-graph strategy.  

\begin{figure}[t]
	\centering
	\subfigure[CiteSeer]{\includegraphics[width=0.242\columnwidth]{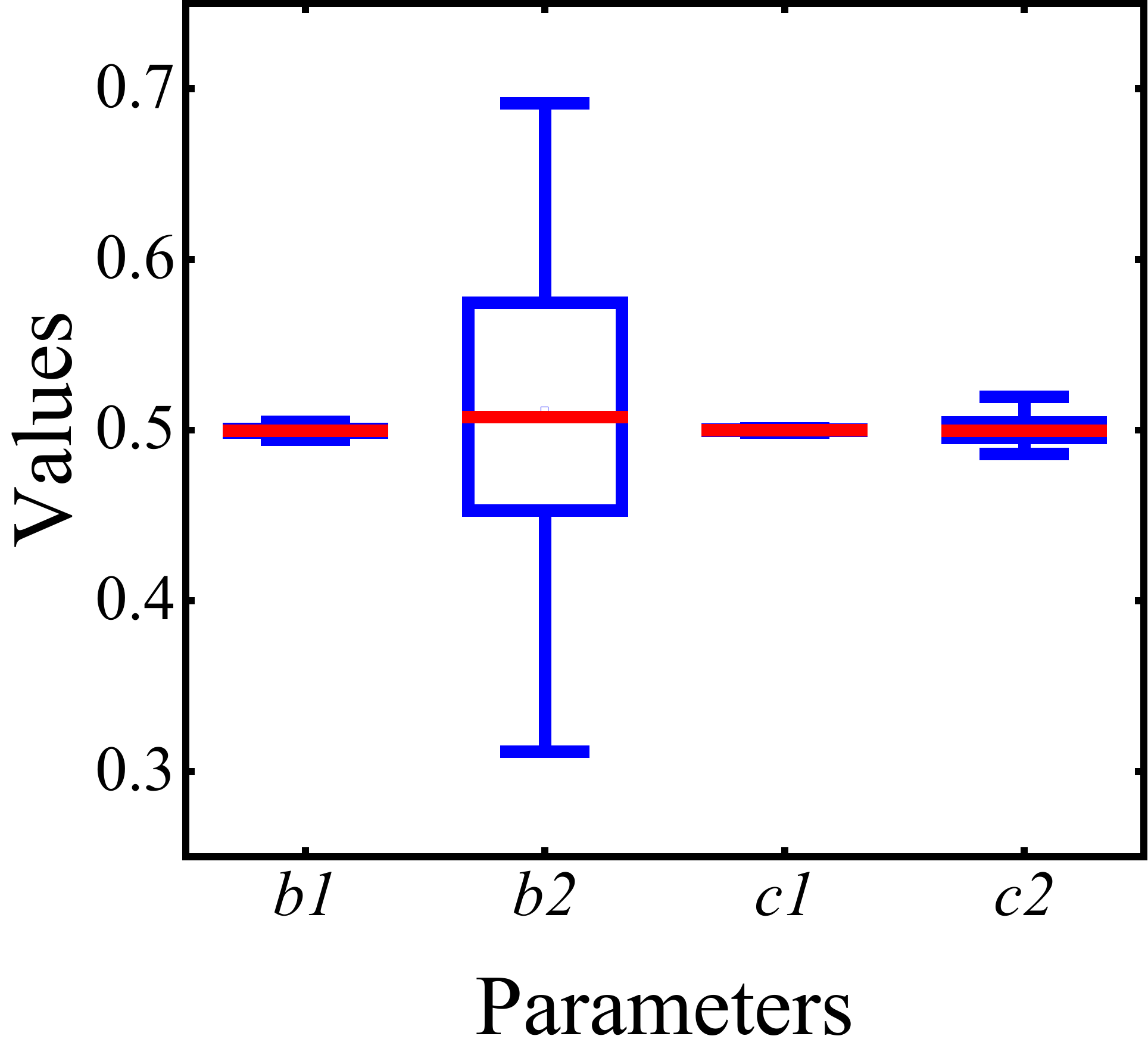}} \subfigure[PubMed]{\includegraphics[width=0.242\columnwidth]{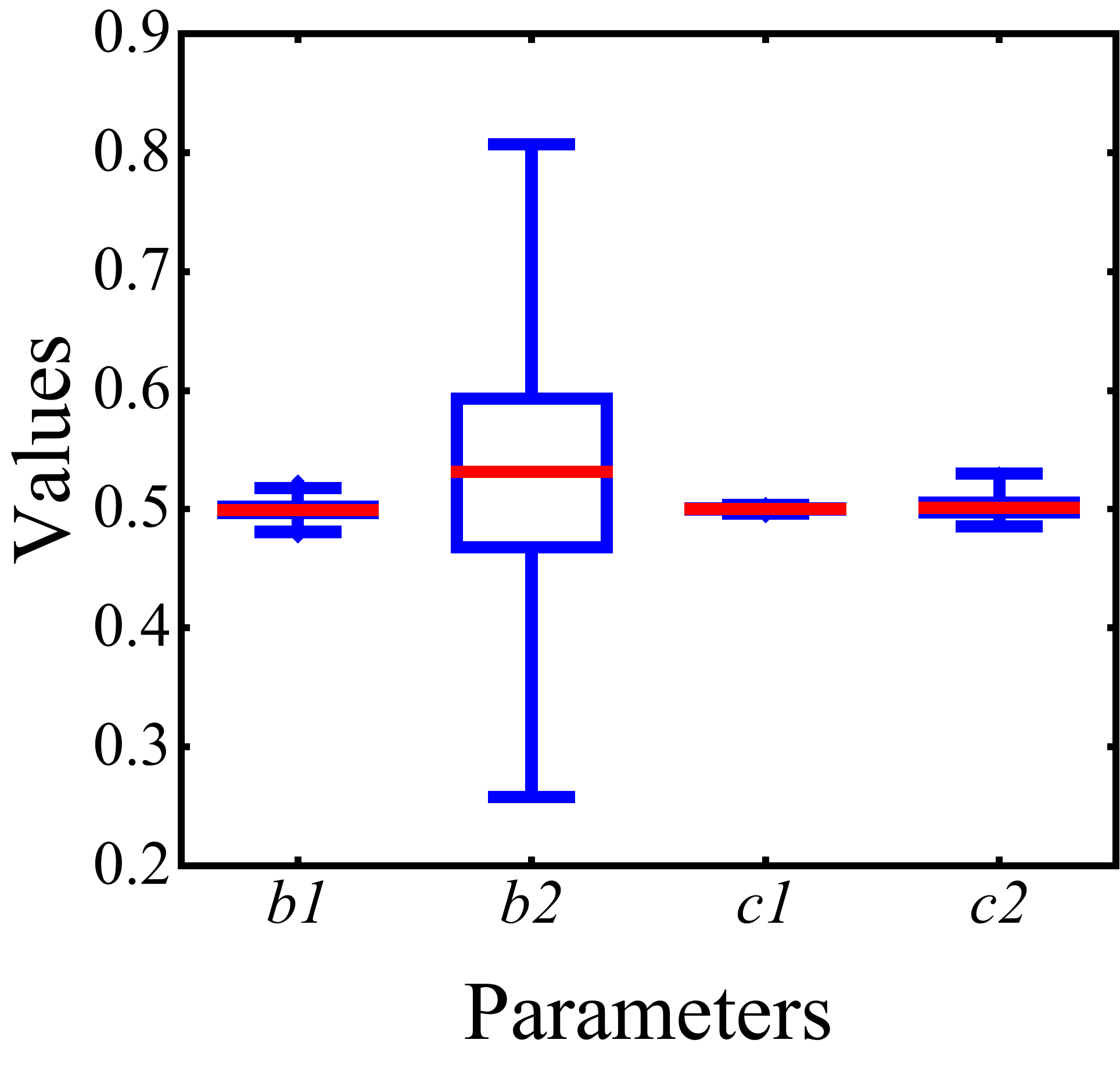}}
	\subfigure[Chameleon]{\includegraphics[width=0.242\columnwidth]{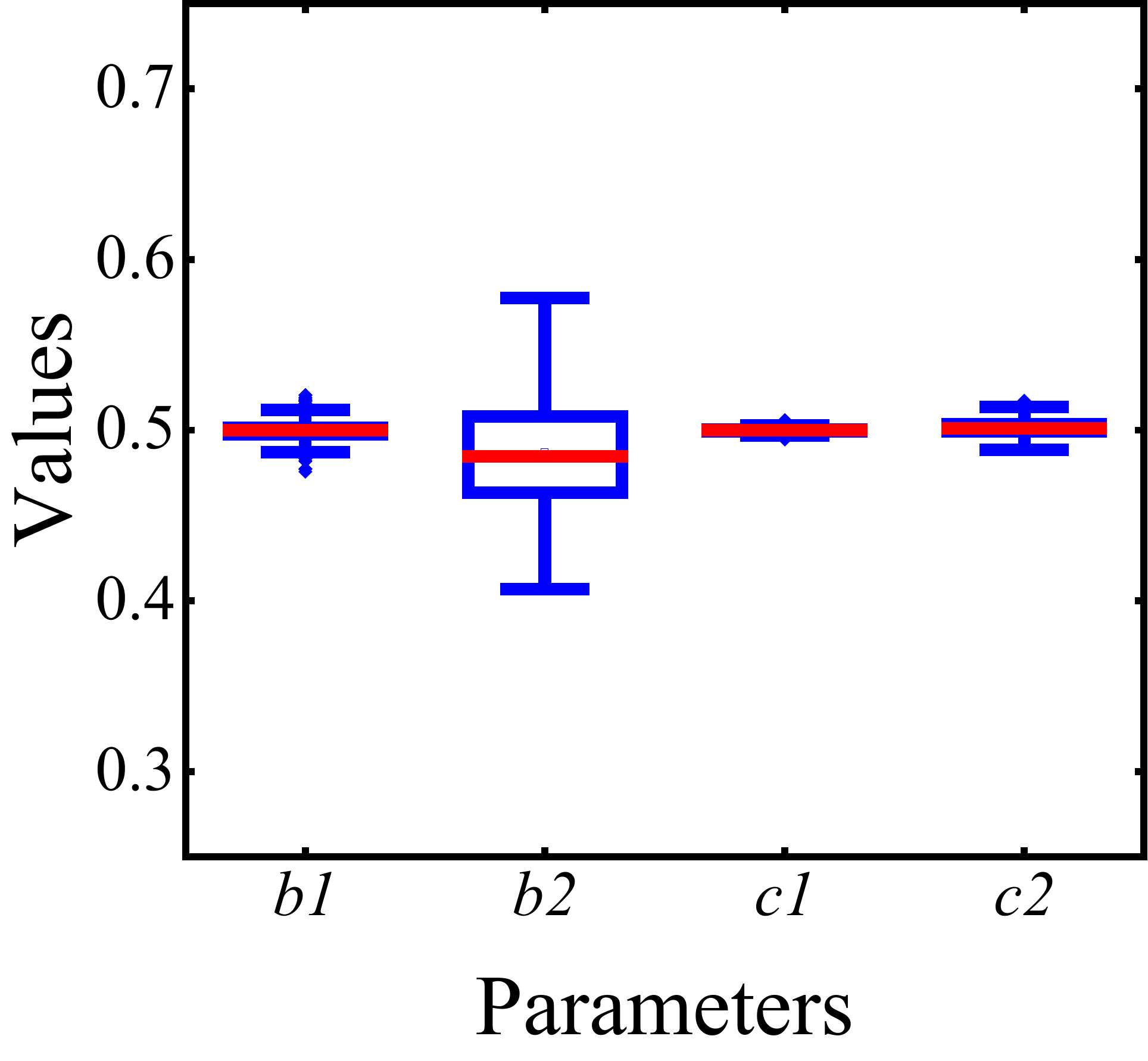}}
	\subfigure[Squirrel]{\includegraphics[width=0.242\columnwidth]{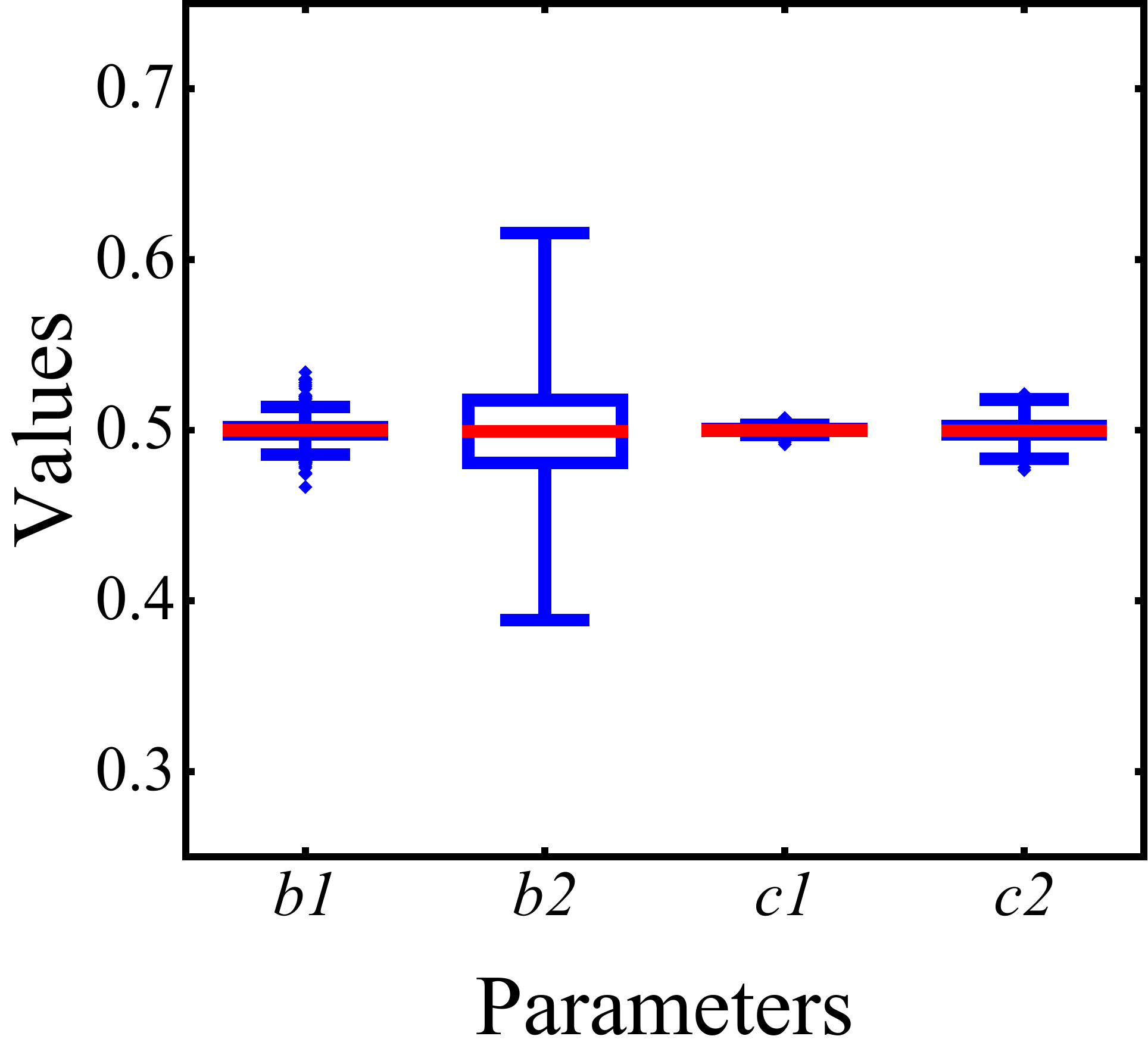}}
	\subfigure[CS]{\includegraphics[width=0.242\columnwidth]{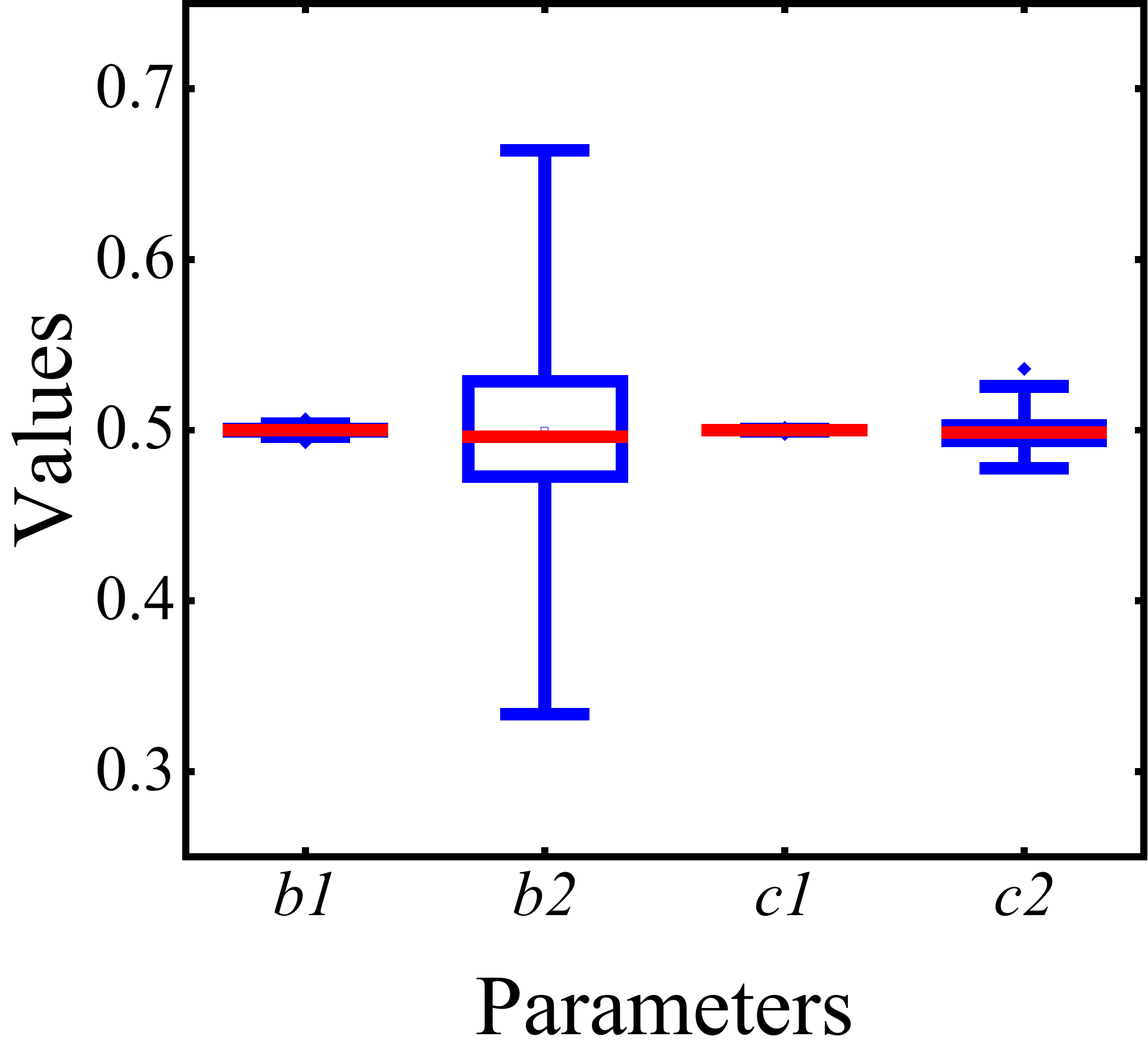}}
	\subfigure[Physics]{\includegraphics[width=0.242\columnwidth]{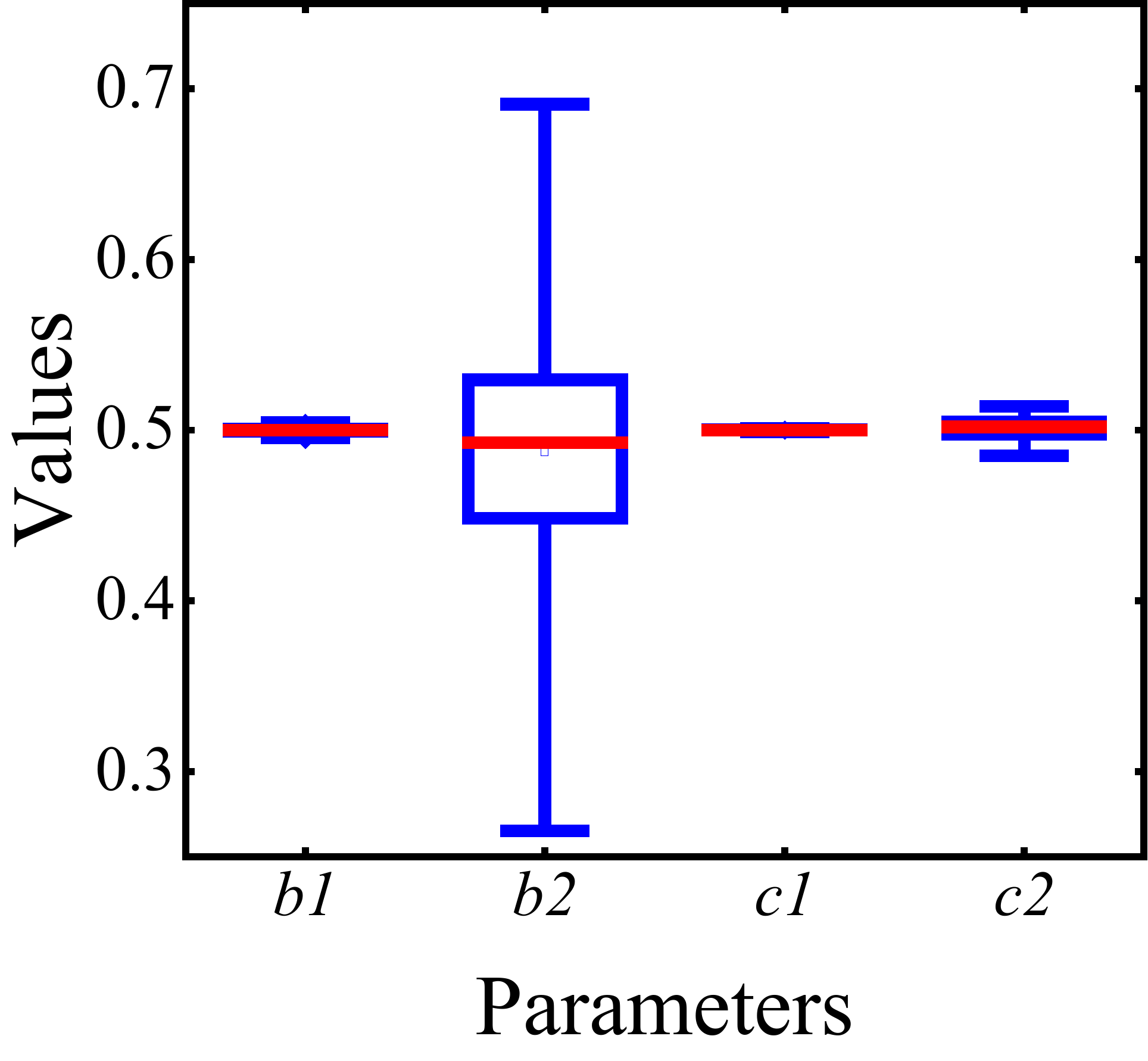}}
	\subfigure[20News]{\includegraphics[width=0.242\columnwidth]{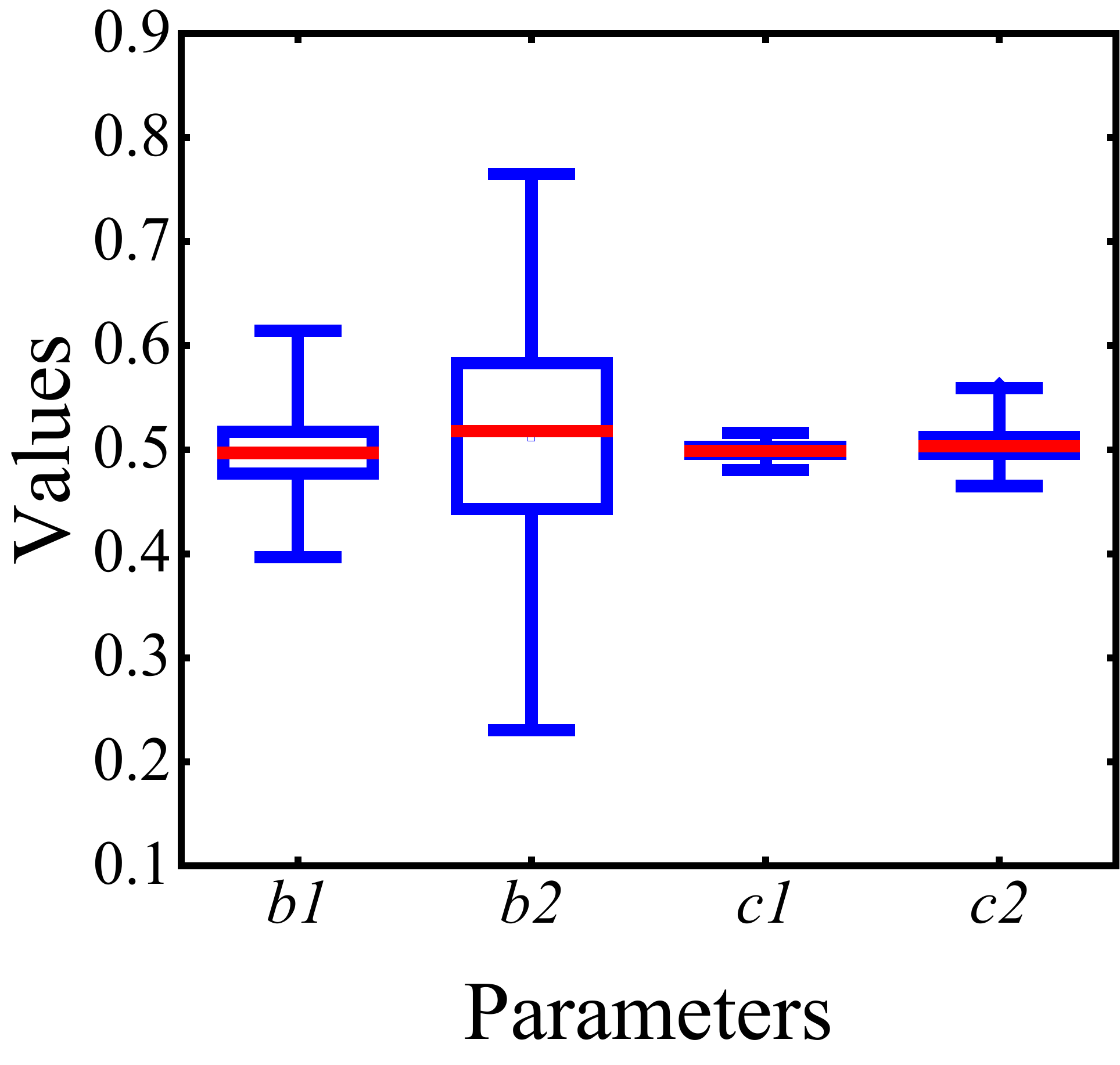}}
	\subfigure[Mini]{\includegraphics[width=0.242\columnwidth]{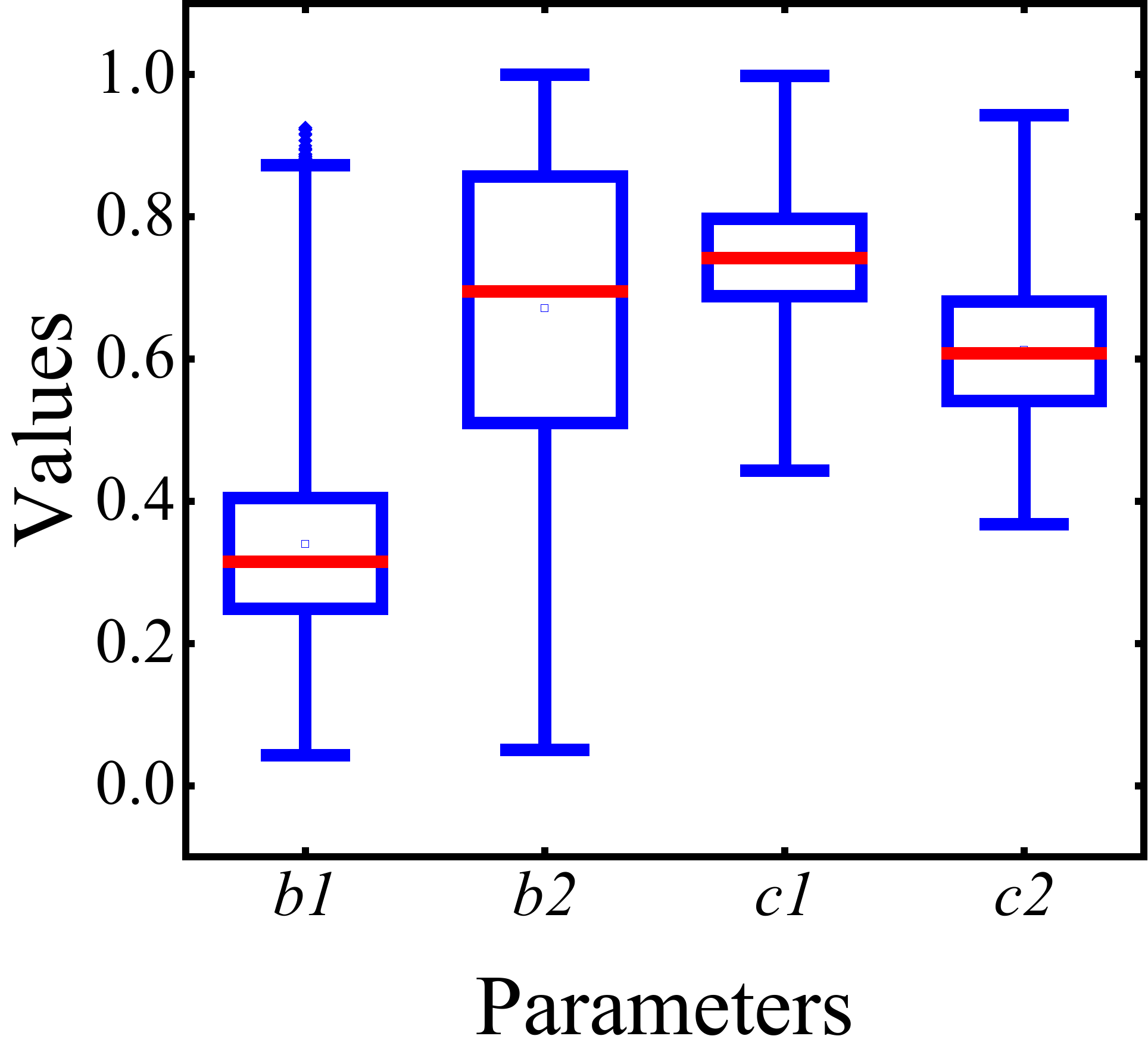}}
	\vspace{-0.5em}
	\caption{Analysis of Gaussian Parameter Distributions of Neural Gaussian Similarity modeling for all datasets.}
	\label{graph2}
	\vspace{-1em}
\end{figure}

\begin{figure}[t]
	\centering
	\subfigure[CiteSeer]{\includegraphics[width=0.242\columnwidth]{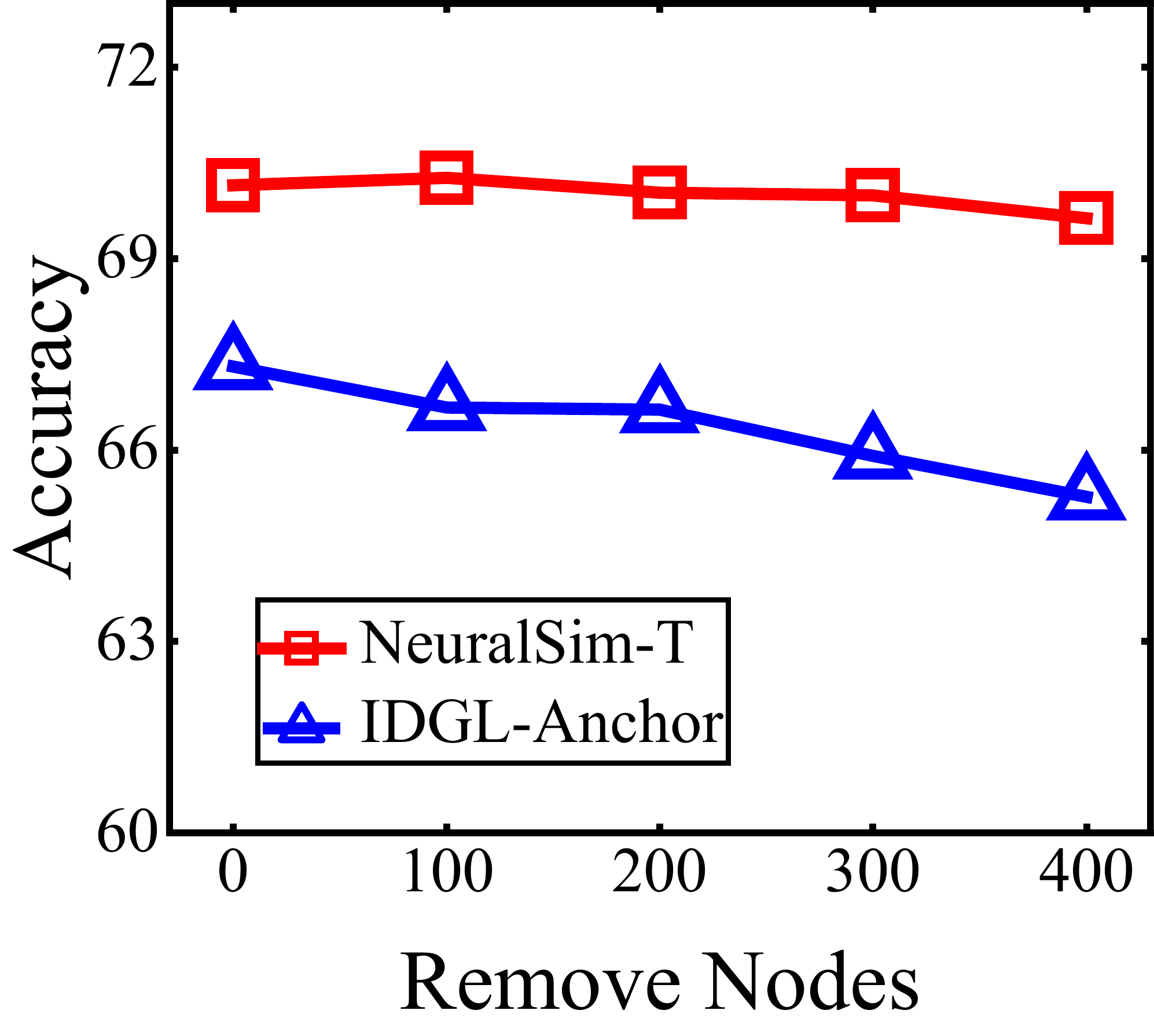}}
	\subfigure[PubMed]{\includegraphics[width=0.242\columnwidth]{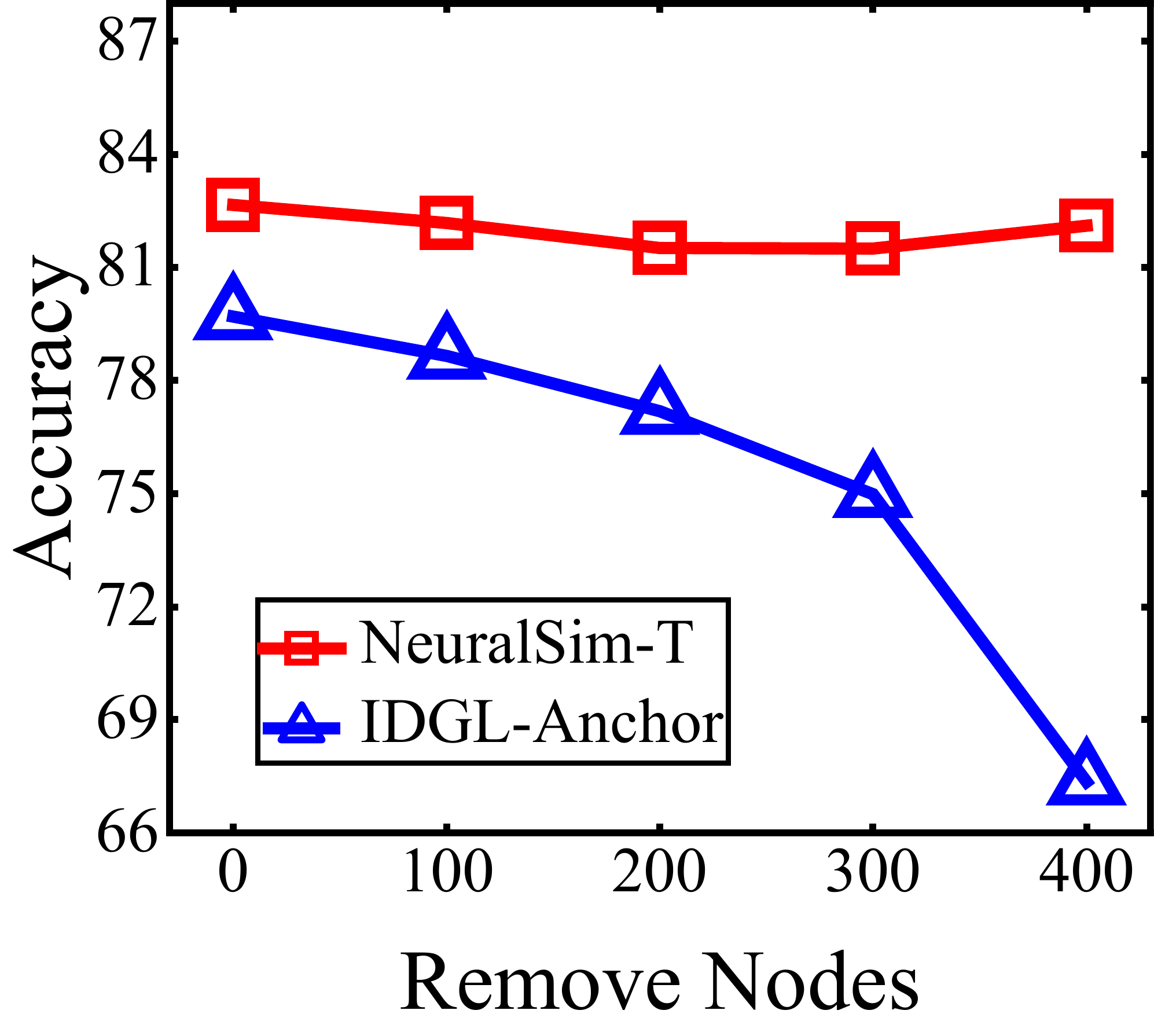}}
	\subfigure[Chameleon]{\includegraphics[width=0.242\columnwidth]{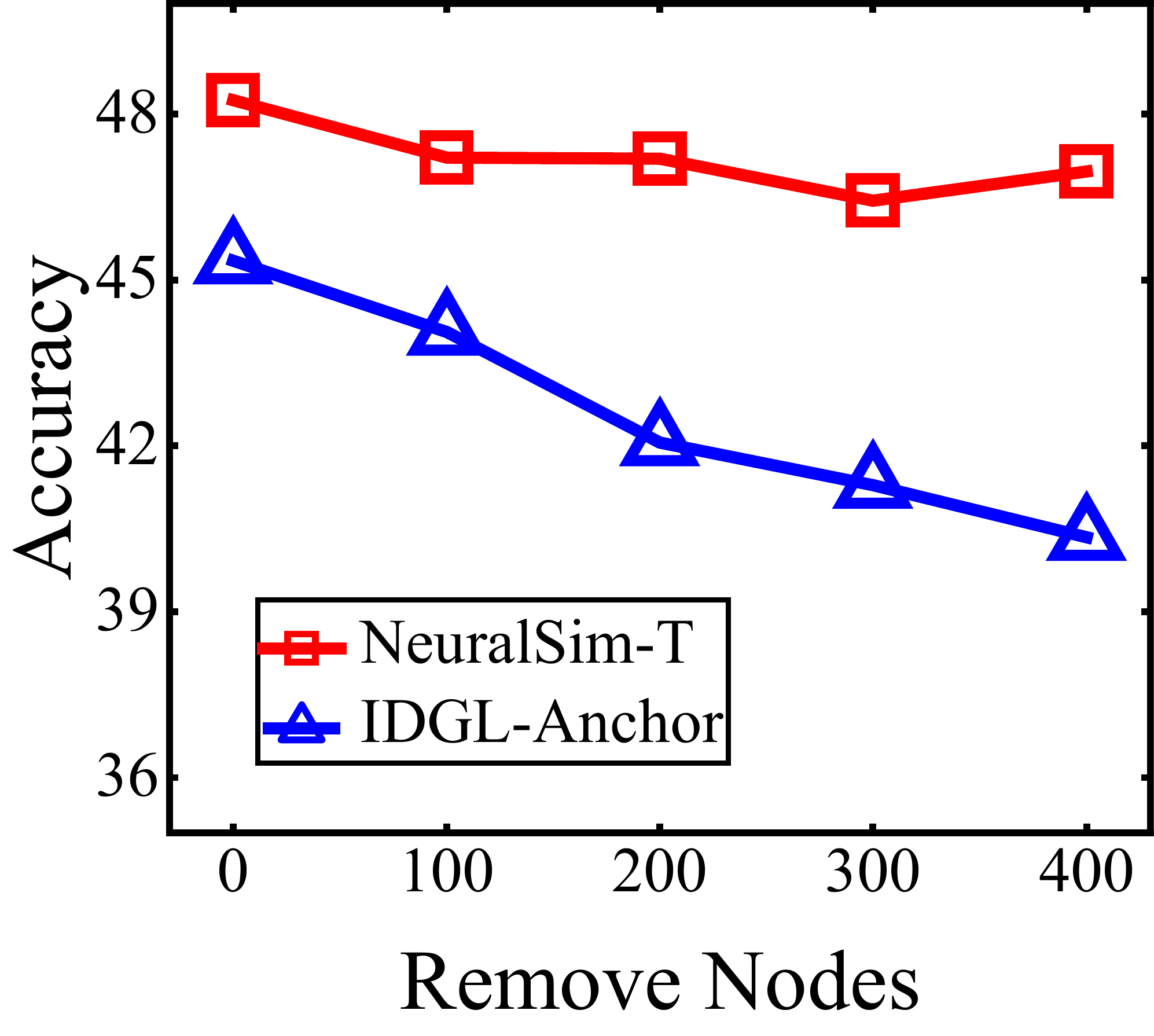}}
	\subfigure[Squirrel]{\includegraphics[width=0.242\columnwidth]{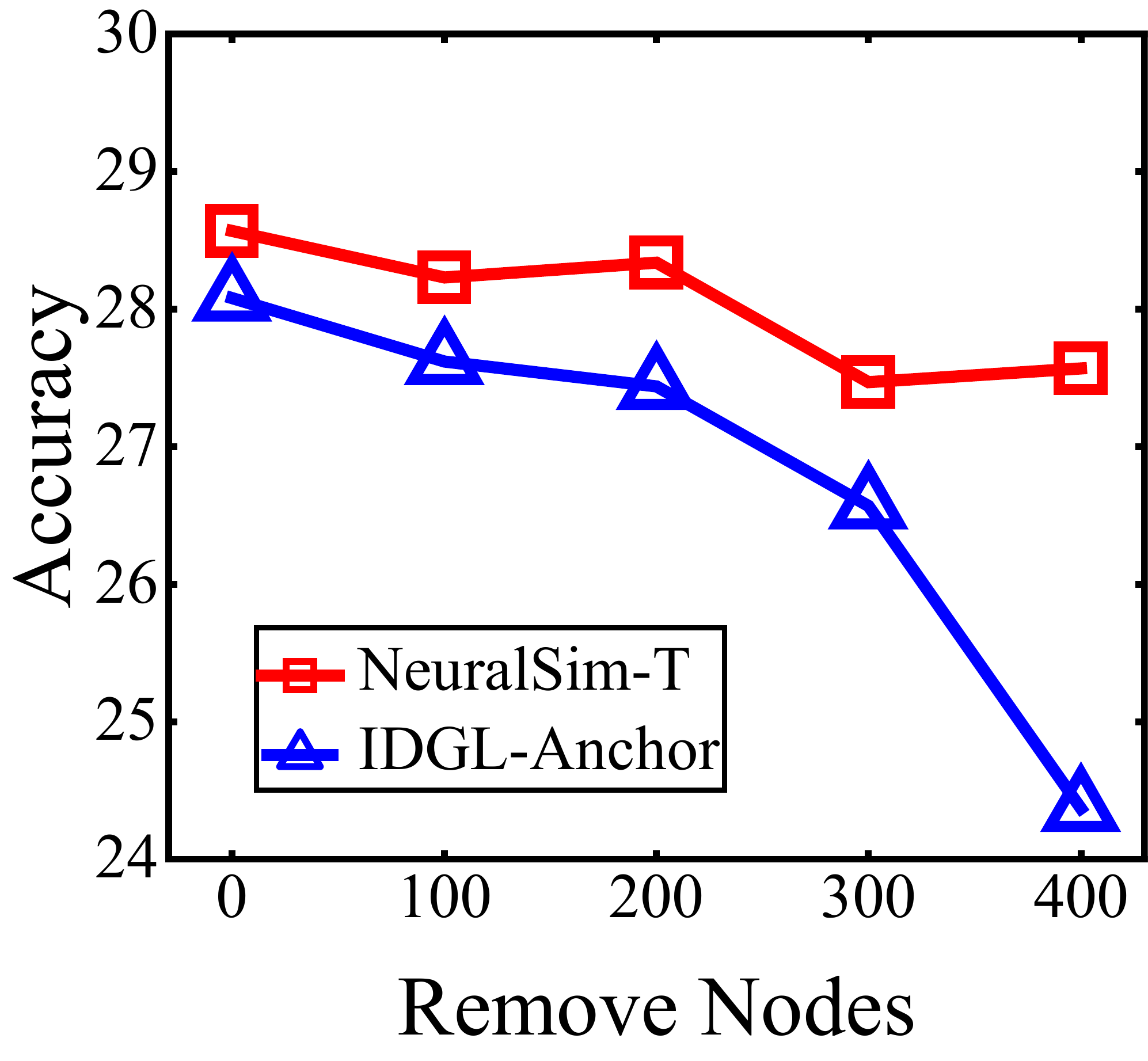}}
	\subfigure[CS]{\includegraphics[width=0.242\columnwidth]{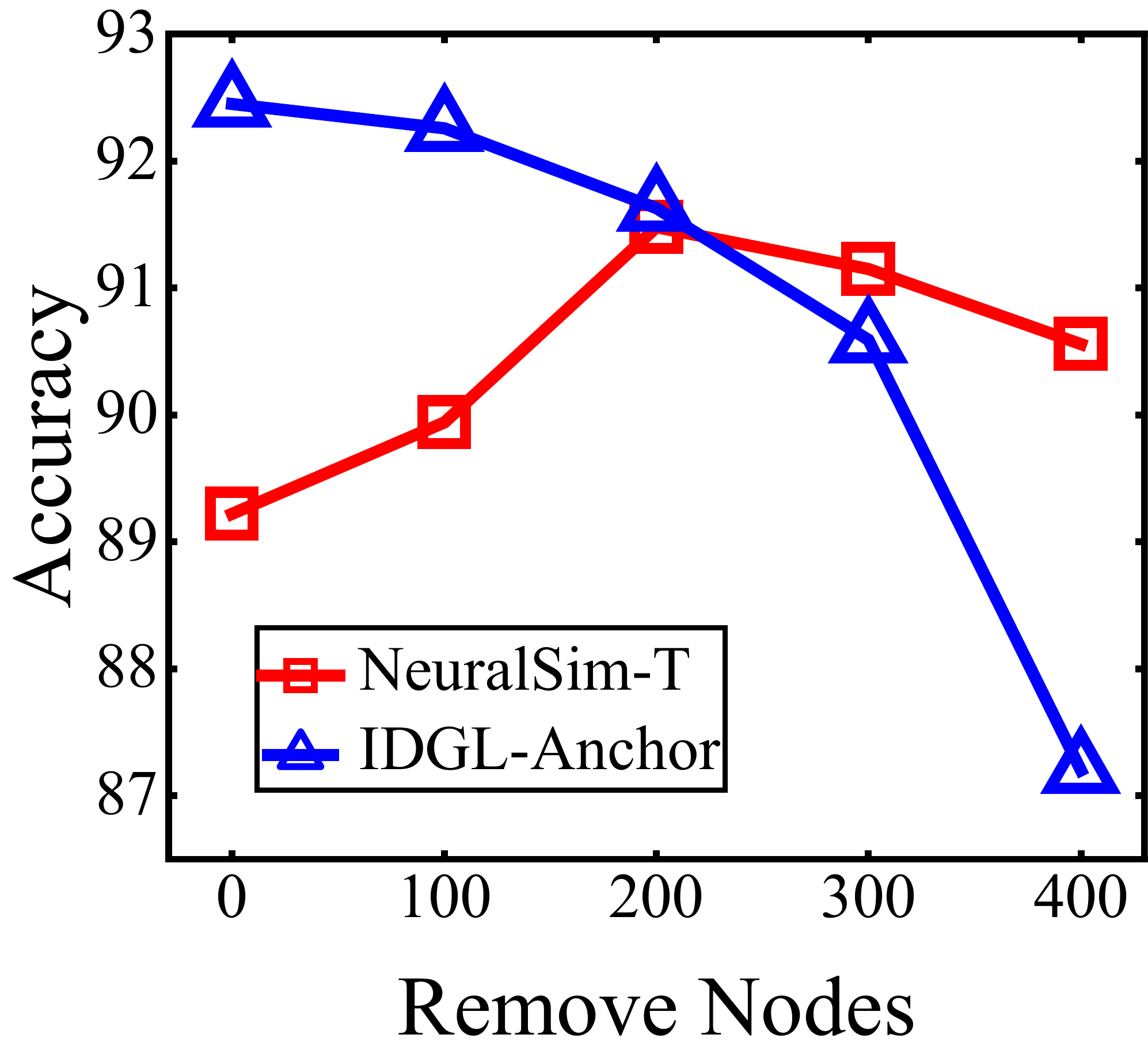}}
	\subfigure[Physics]{\includegraphics[width=0.242\columnwidth]{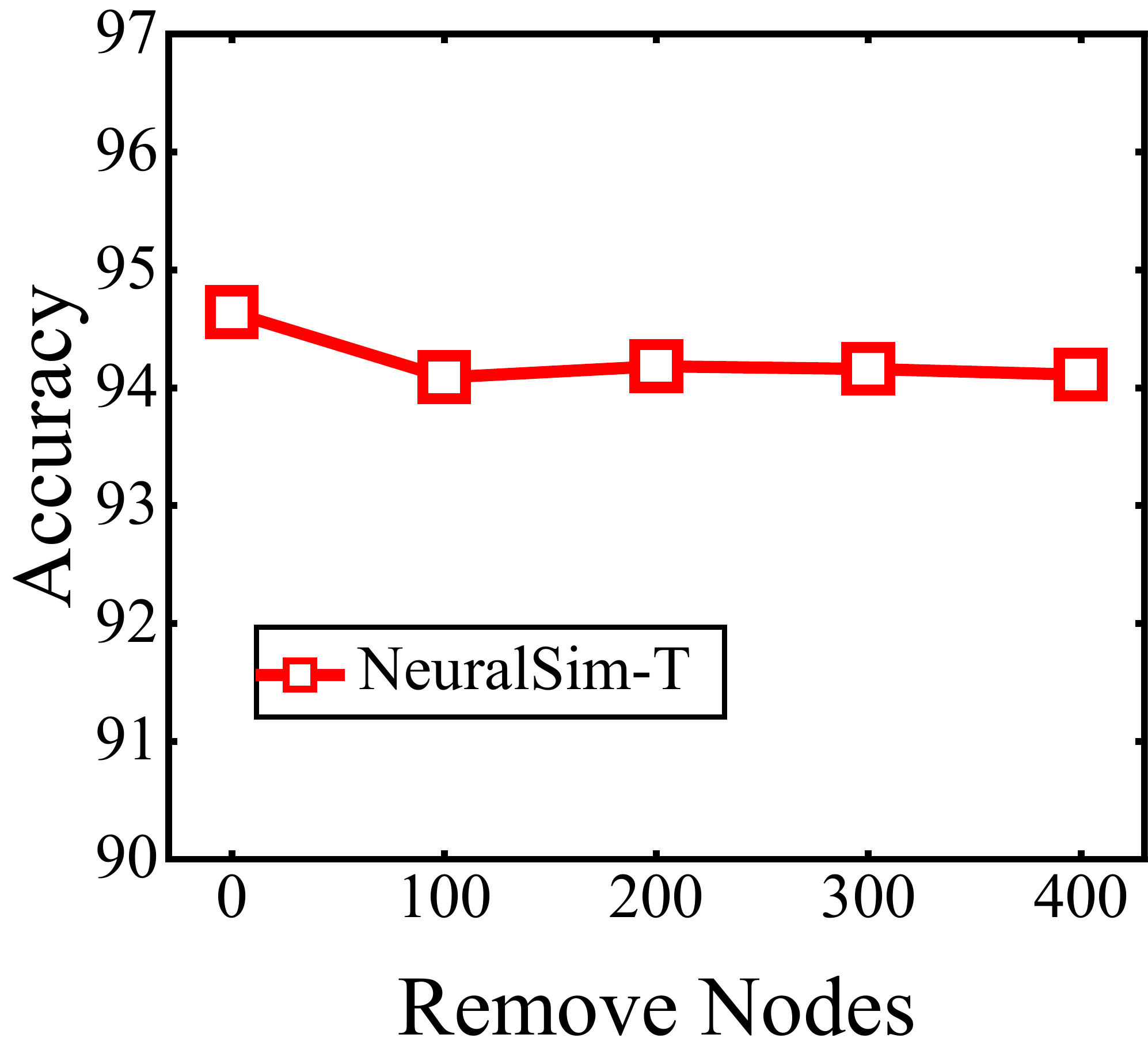}}
	\subfigure[20News]{\includegraphics[width=0.242\columnwidth]{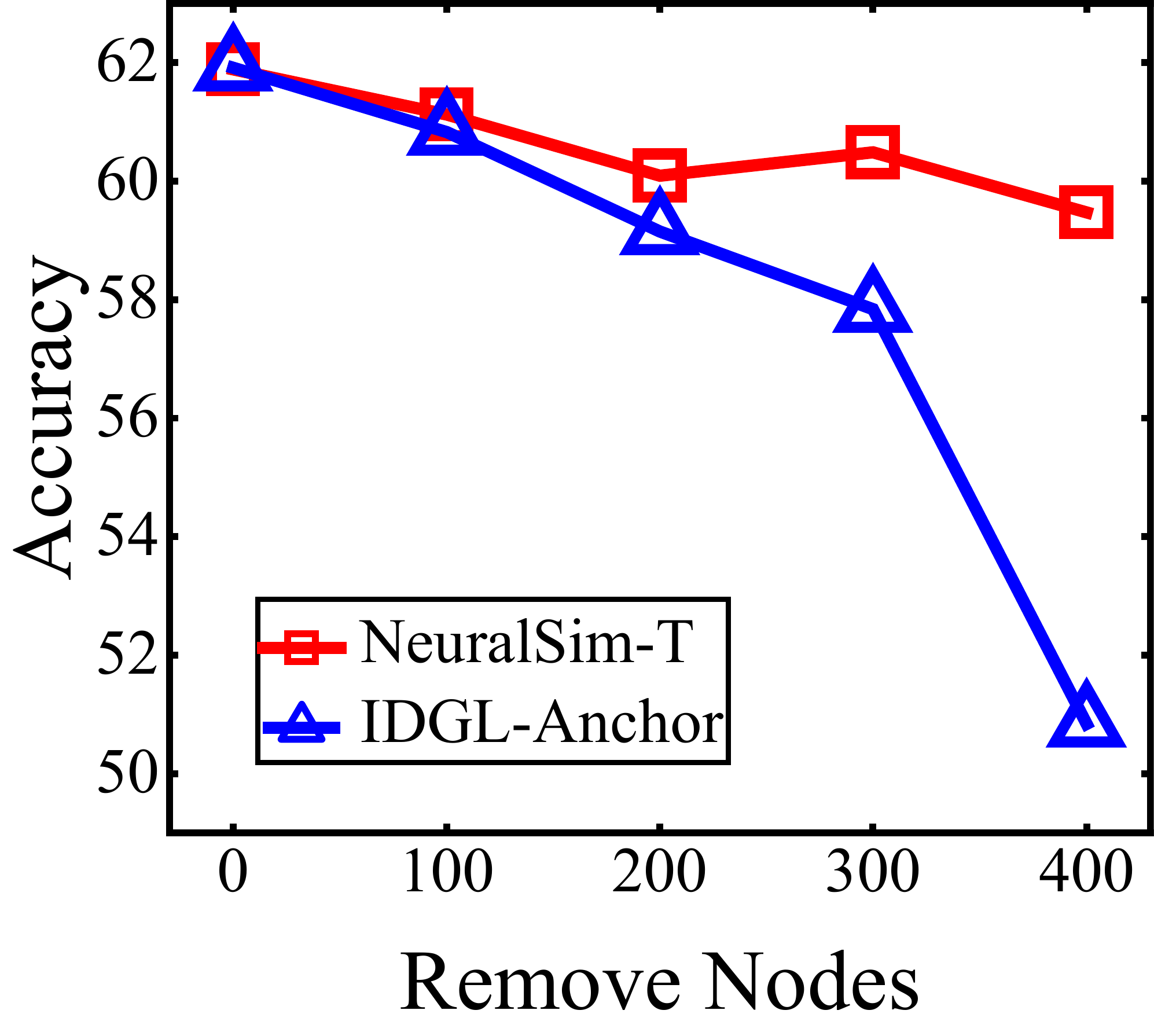}}
	\subfigure[Mini]{\includegraphics[width=0.242\columnwidth]{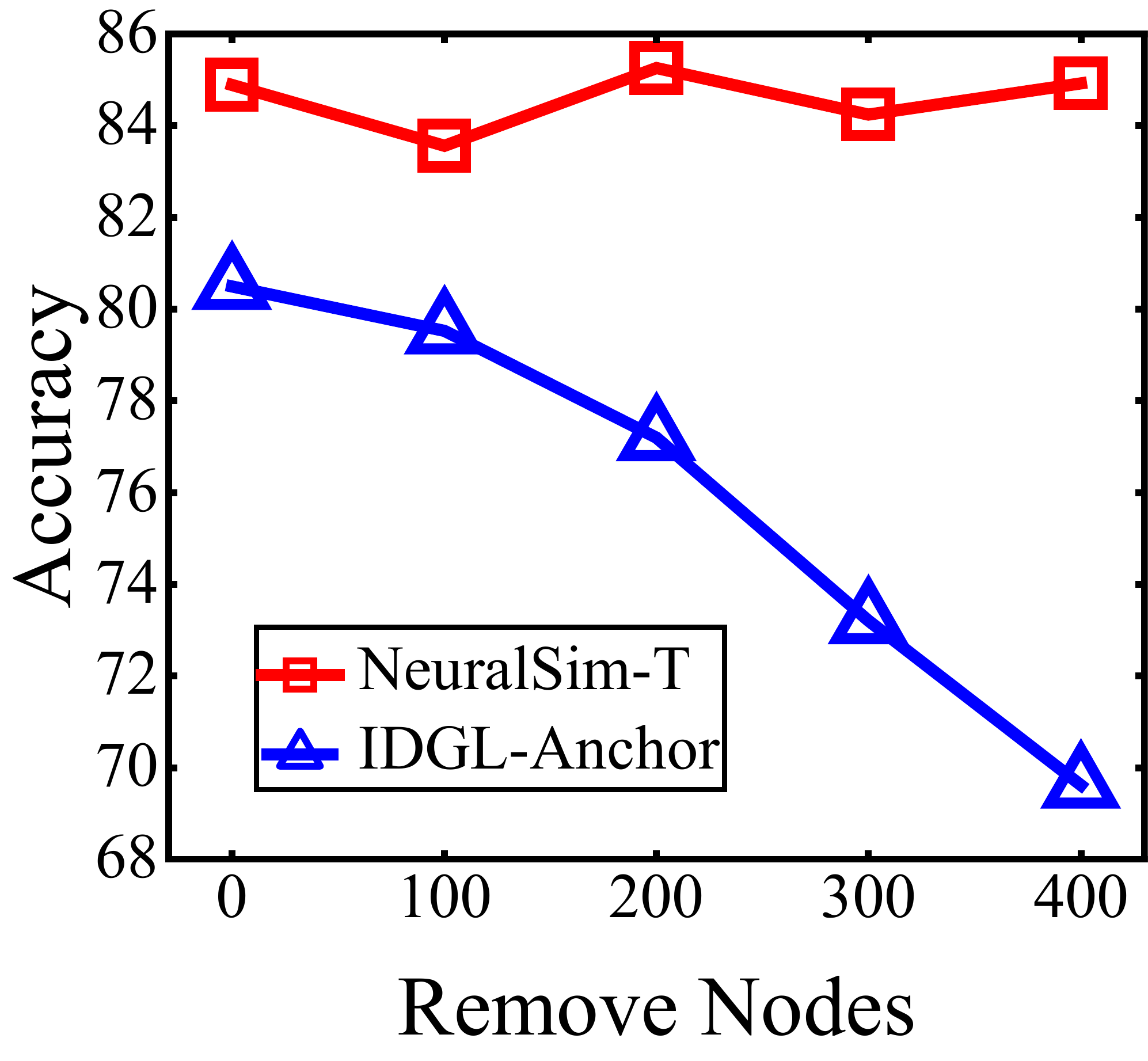}}
	\vspace{-0.5em}
	\caption{Analysis of removing transition-graph and anchor-graph nodes for NeuralGauSim-T and IDGL-Anchor, respectively.}
	\label{graph1}
	\vspace{-1em}
\end{figure}

\subsection{Results Comparison of Different Edge Samples}
To evaluate the performance of different number of sampling edges $\mathcal{K}$, we perform comparative experiments by setting $\mathcal{K} \in \{5, 10, 15, 20\}$ for full-graph and transition-graph settings. Experimental results are reported in Table \ref{tab2} and Table \ref{tab3}, respectively. From these two tables, we have the following observations. First, compared with the Linear Similarity (LinSim), the proposed Gaussian Similarity (GauSim) and Neural Gaussian Similarity (NeuralGauSim) methods consistently achieve better performance under different number of edge samples, indicating the effectiveness of the proposed Gaussian similarity modeling. Second, by comparing Gaussian Similarity, i.e., GauSim and GauSim-T, and Neural Gaussian Similarity, i.e., NeuralGauSim and NeuralGauSim-T, we can find that the NeuralGauSim method exhibits superior performance except when $\mathcal{K} = 10$ on 20News dataset, demonstrating the effectiveness of the learnable parameters of the Gaussian similarity. Third, by comparing performance of different number of edge samples, we can observe that as the parameter $\mathcal{K}$ increases, the performance of Linear Similarity (LinSim) deteriorates. In contrast, the proposed Gaussian Similarity (GauSim) and Neural Gaussian Similarity (NeuralGauSim) methods consistently maintain its performance level, further substantiating the effectiveness of the proposed methods.   

\subsection{Analysis of Gaussian Parameter Distributions}
To investigate whether the parameter values of Gaussian similarity learned by the proposed similarity modeling method are meaningful, we conduct the Gaussian parameter distribution analysis experiments on all datasets. Experimental results are shown in Figure \ref{graph2}. Specifically, we use the boxplot to count the distribution of parameters $b$ and $c$ of the two-layer model, where the parameters of the first layer are denoted as $b_1$, $c_1$, and the second layer are denoted as $b_2$, $c_2$. For clarity, we multiply all parameters $c$ by a scale factor of $10$. From these figures, we have the following observations. First, the adjustment range of the parameters of the second layer is larger than that of the first layer. Second, for all datasets, the parameter $b_2$ exhibits a wider distribution compared to other parameters, implying that adjusting the peak of the bell-shaped Gaussian similarity function is more important than adjusting the slope of the similarity function.

\subsection{Sensitivity of Transition Graph Nodes}

\begin{figure}[t]
	\centering
	\subfigure[CiteSeer]{\includegraphics[width=0.242\columnwidth]{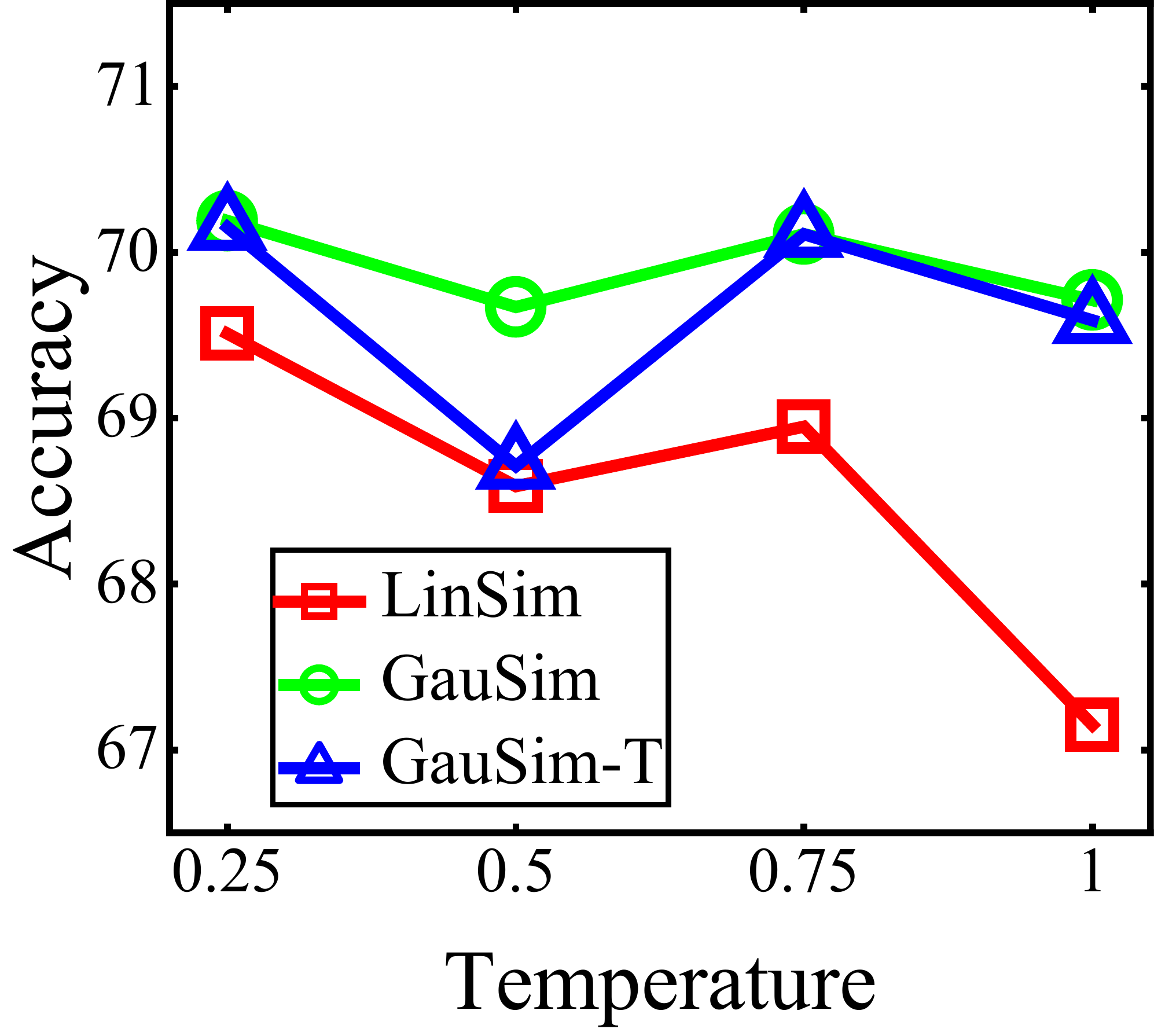}}
	\subfigure[PubMed]{\includegraphics[width=0.242\columnwidth]{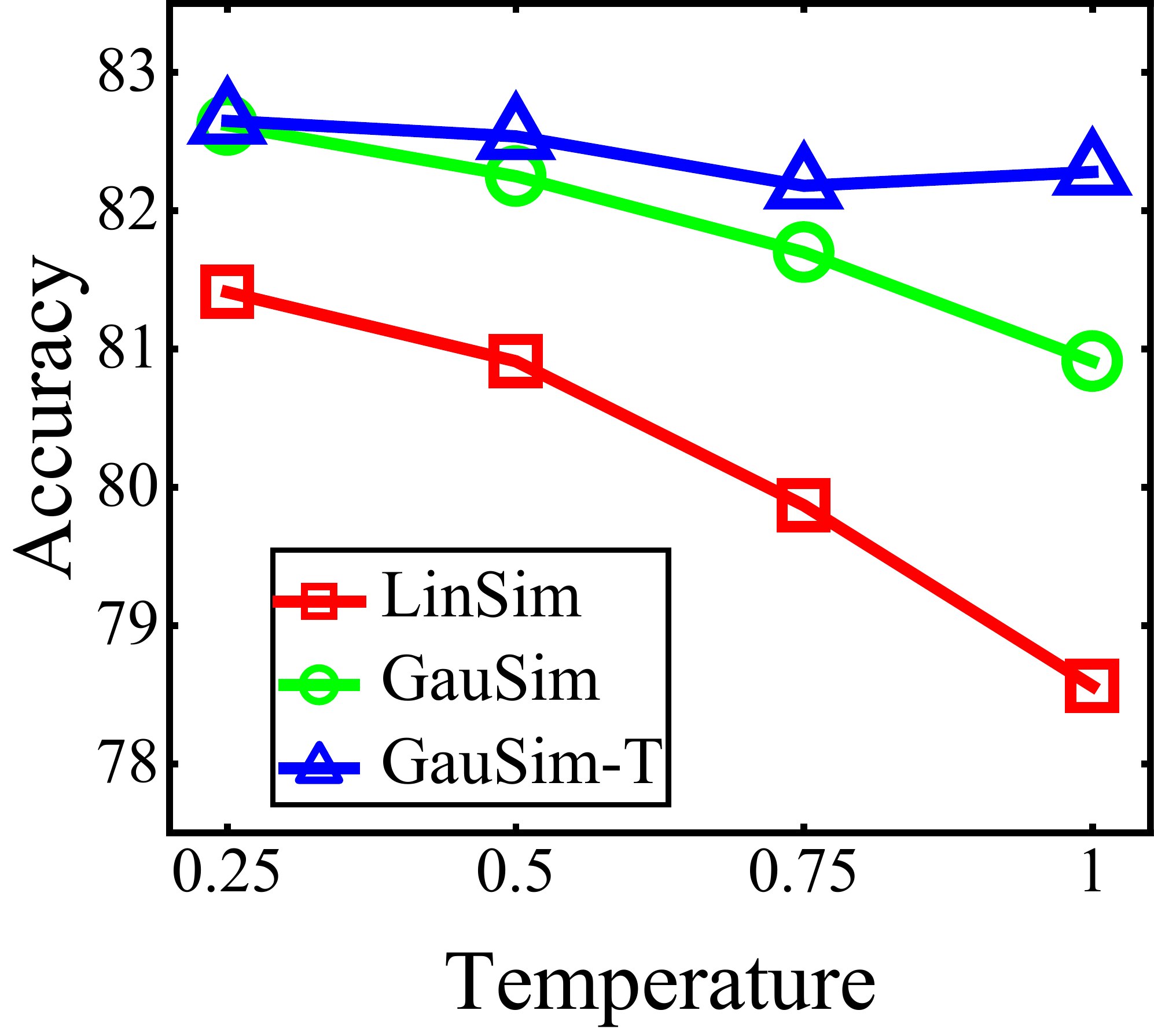}}
	\subfigure[Chameleon]{\includegraphics[width=0.242\columnwidth]{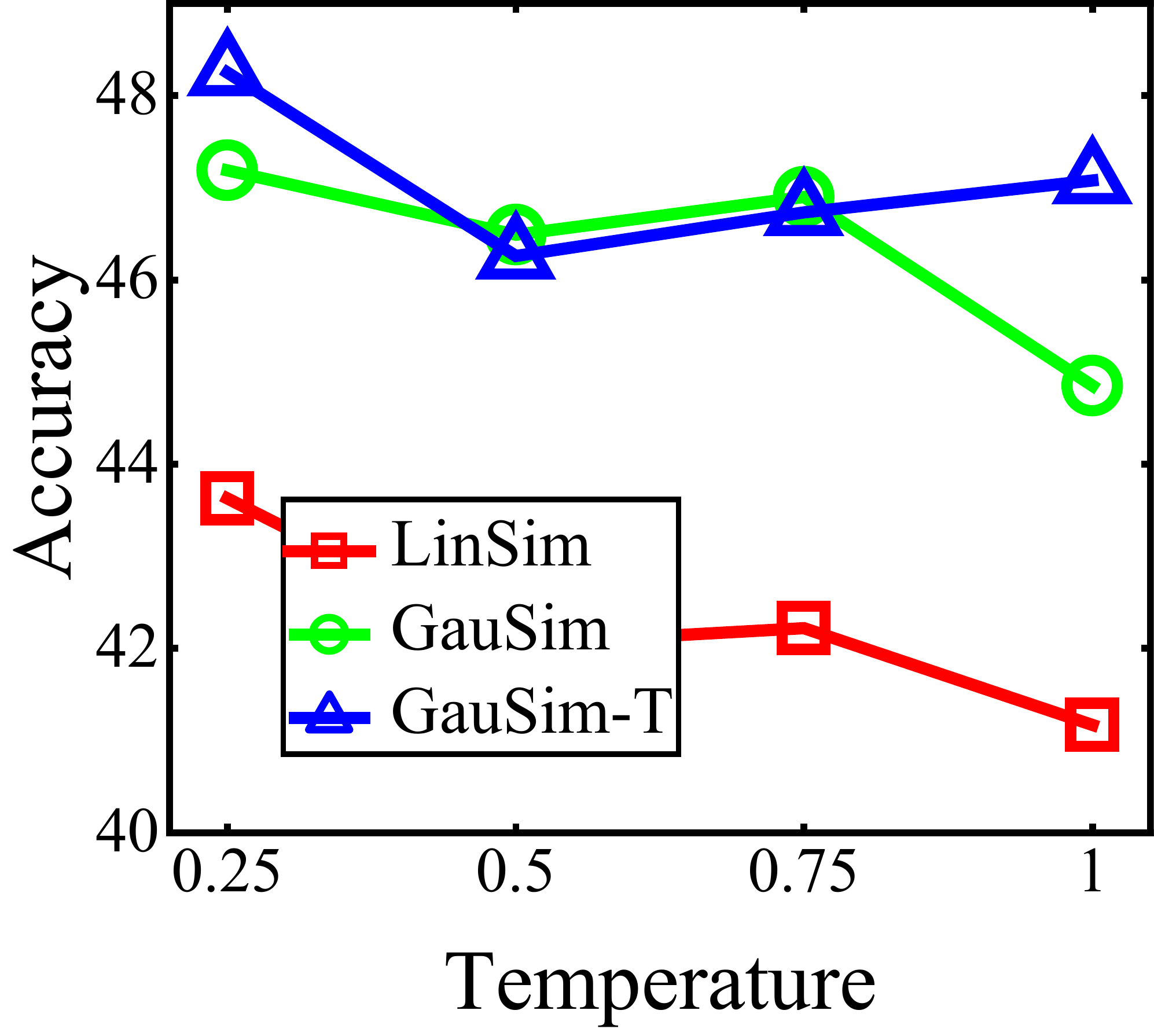}}
	\subfigure[Squirrel]{\includegraphics[width=0.242\columnwidth]{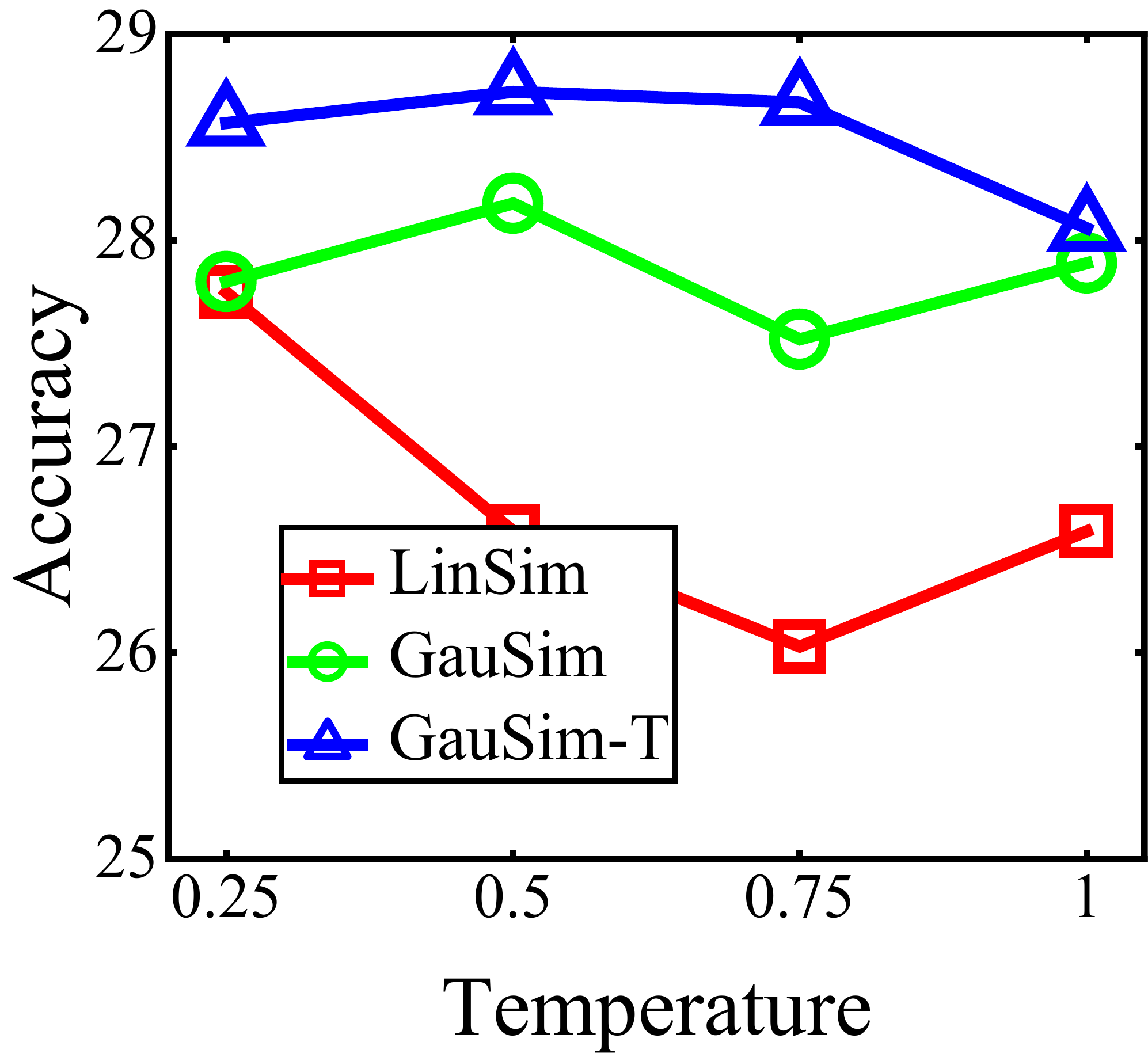}}
	\subfigure[CS]{\includegraphics[width=0.242\columnwidth]{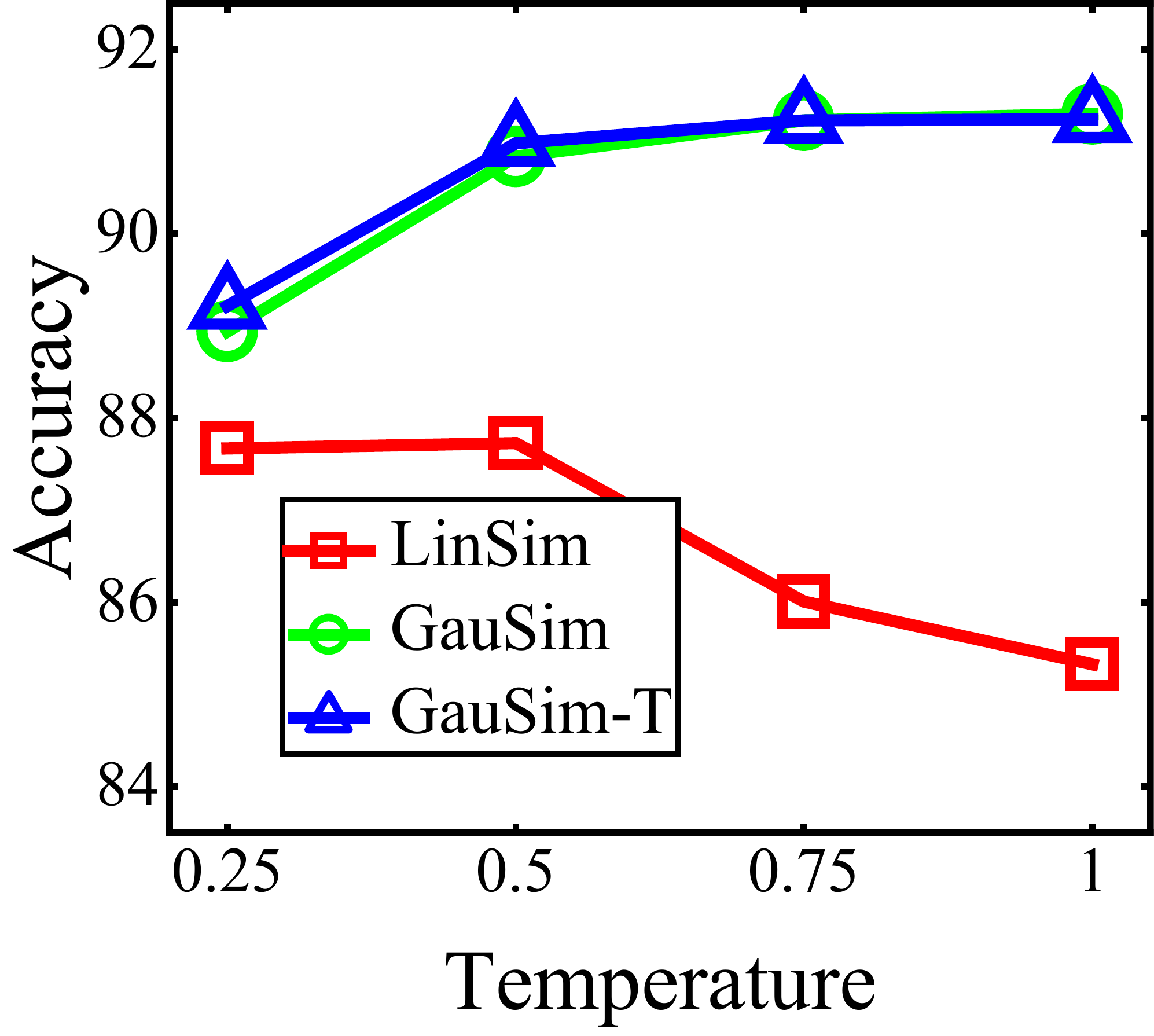}}
	\subfigure[Physics]{\includegraphics[width=0.242\columnwidth]{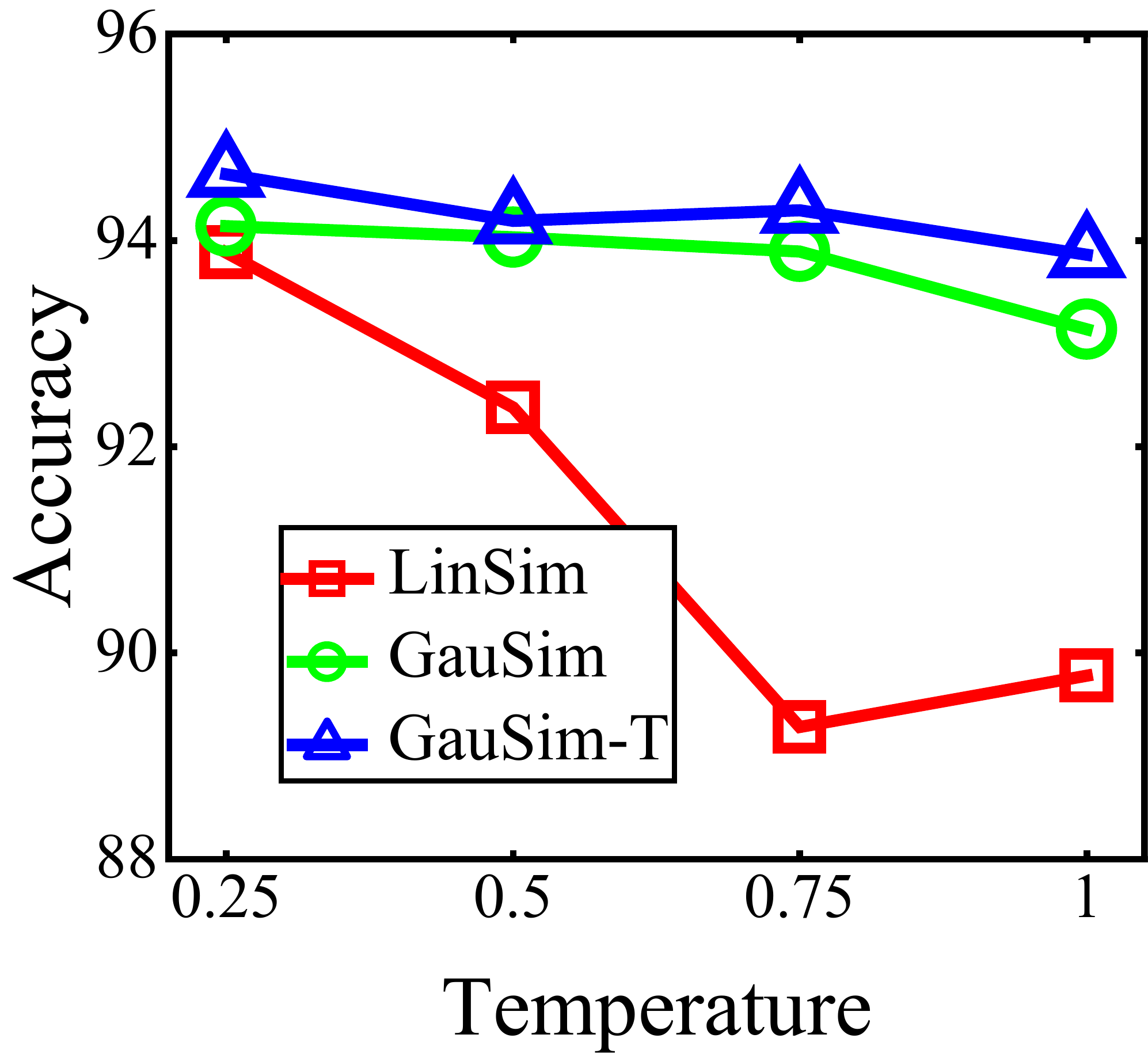}}
	\subfigure[20News]{\includegraphics[width=0.242\columnwidth]{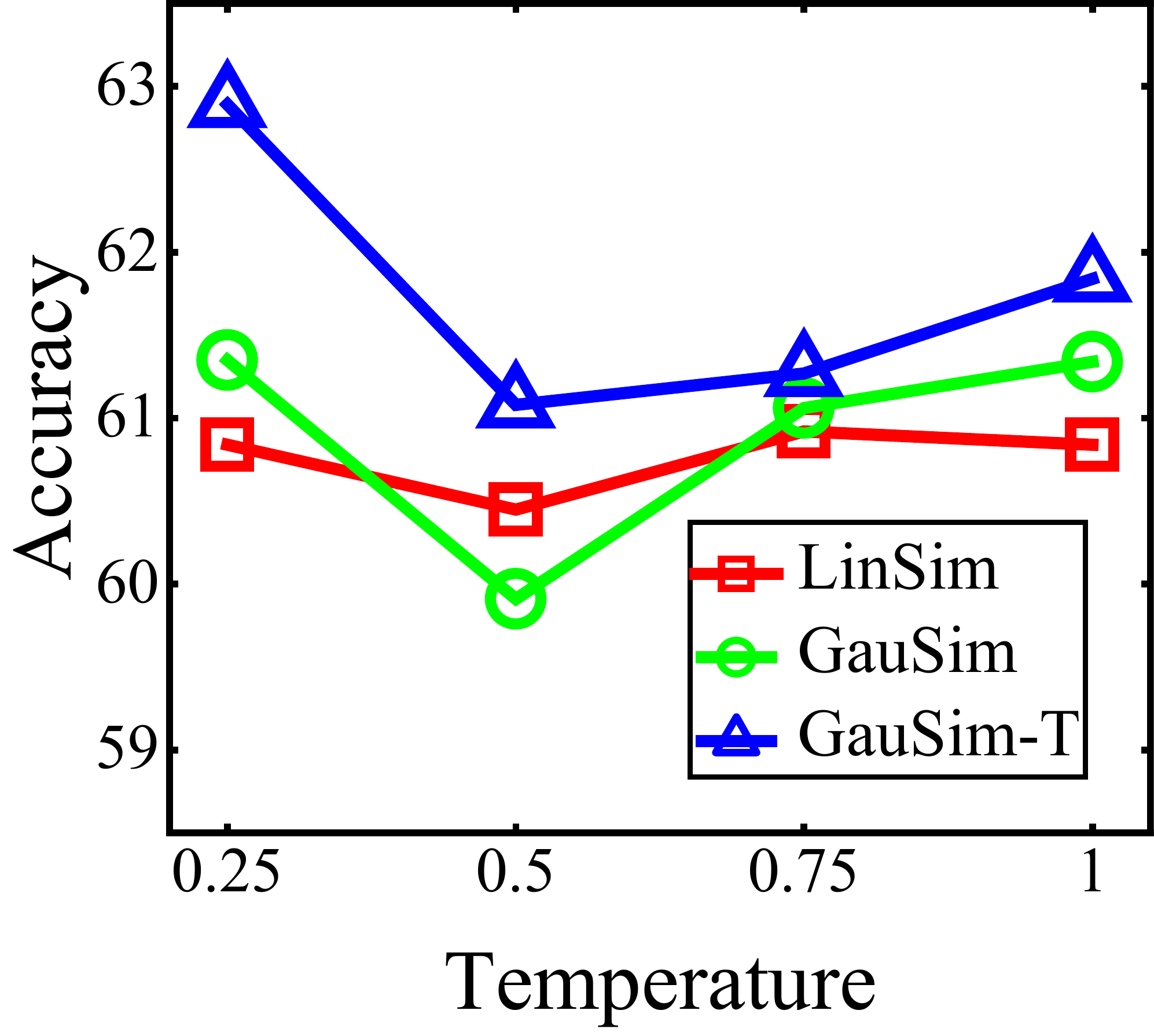}}
	\subfigure[Mini]{\includegraphics[width=0.242\columnwidth]{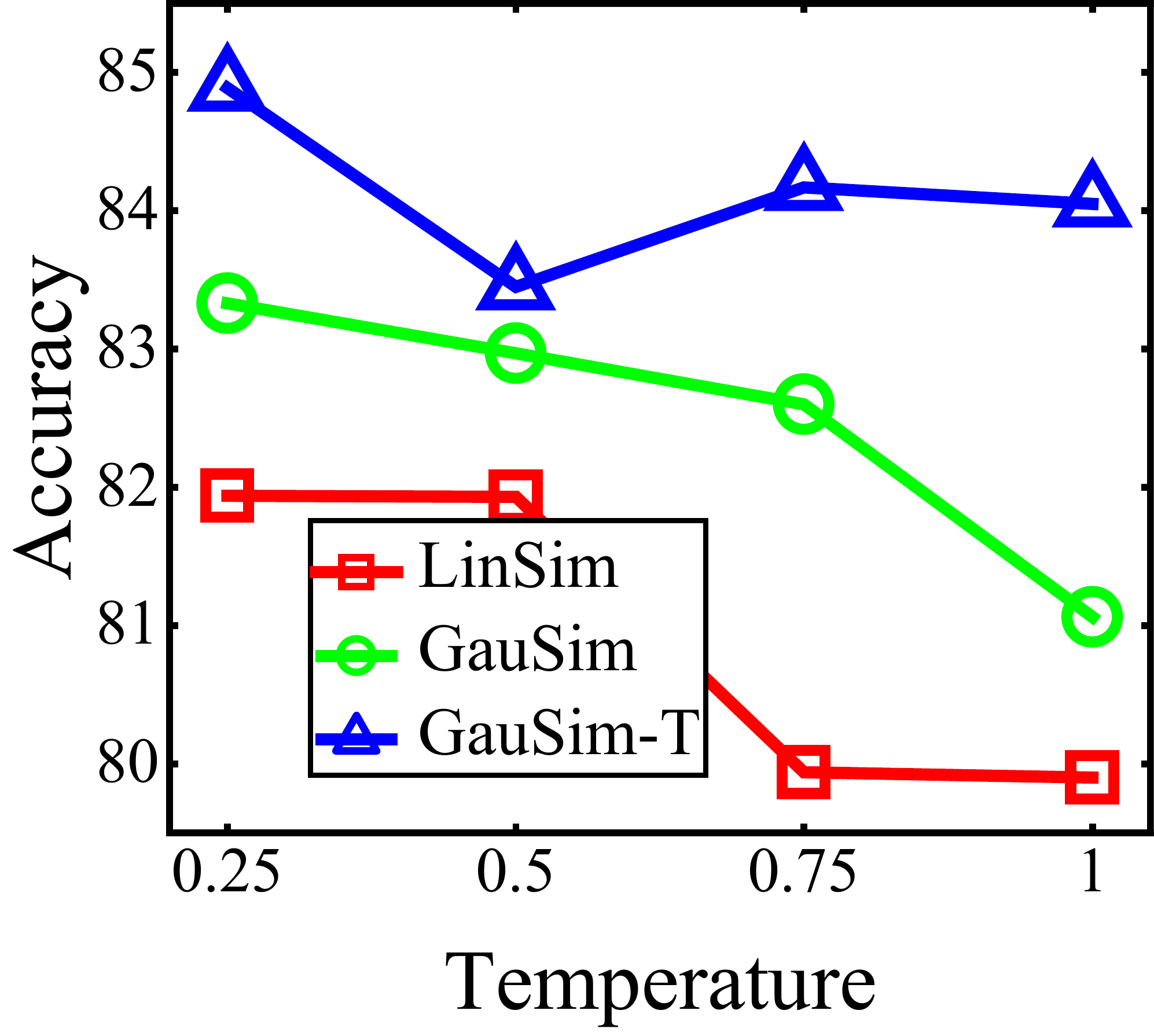}}
	\vspace{-1em}
	\caption{Impact of Temperature Parameter $\tau$.}
	\label{graph4}
\end{figure}

\begin{table}[t]
	\centering
	\setlength{\tabcolsep}{9mm}
	\scalebox{0.6}{
		\begin{tabular}{|c|c|c|}
			\toprule
			&Training Time (s)&GPU Memory (MB)\\
			\midrule
			\midrule
			IDGL&0.7967&4,358\\
			IDGL-Anchor&0.1249&2,986\\
			NodeFormer&\textbf{0.0463}&\textbf{1,248}\\
			LinSim&0.1145&4,486\\
			LinSim-T&\textbf{0.0367}&\textbf{1,796}\\
			GauSim&0.1193&4,618\\
			GauSim-T&\textbf{0.0382}&\textbf{1,816}\\
			NeuralGauSim&0.1238&4,664\\
			NeuralGauSim-T&\textbf{0.0393}&\textbf{1,821}\\
			\bottomrule
	\end{tabular}}
	\caption{Comparison of training time (s) and GPU memory (MB) cost on CiteSeer dataset per epoch.}
	\label{tab4} 
	\vspace{-1em}
\end{table}

In order to study the effect of different number of transition-graph nodes, we conducted a comparative \mbox{experiment} to observe the change of model performance by removing the number of transition-graph nodes for the proposed Transition-graph based Neural Gaussian Similarity method (NeuralGauSim-T). Here, for the initial transition graph with $500$ nodes, we remove $\{0, 100, 200, 300, 400\}$ \mbox{nodes} respectively. Similarly, for the IDGL-Anchor method, we also remove the same anchor graph nodes, i.e, $\{0, 100, 200, 300, 400\}$, for performance comparison purpose. Experimental results are shown in Figure \ref{graph1}. Note that for Physics dataset, we only report the results of NeuralGauSim-T since the previous IDGL-Anchor method exceeds the memory limit due to the initial fully kNN graph requires $\mathcal{O}(n^2)$ memory consumption. From these figures, we can observe that for NeuralGauSim-T, removing nodes does not bring about a significant decrease in model performance, and even an improvement in performance. Conversely, for the IDGL-Anchor method, removing nodes brings significant performance degradation, especially when the number of removed nodes tends to be large. The primary justification lies in the fact that the anchor graph is generated through random sampling from the original graph, whereas the construction of the transition-graph relies on model learning.

\section{Impact of Temperature Parameter}
To evaluate the impact of different temperature parameter $\tau$ of Gumbel-Softmax distribution, we conduct a comparative experiment by setting different temperature parameter $\tau \in \{0.25, 0.5, 0.75, 1\}$ on all datasets. Experimental results are shown in Figure \ref{graph4}. This figure illustrates the direct impact of modifying the temperature parameter on the performance of the proposed model. Different datasets may require different temperature parameter setting strategies. Note that the previous work \cite{neuralspar} suggests that using an annealing strategy for the temperature parameter. In this paper, following NodeFormer \cite{nodeformer}, we set the same temperature parameters $\tau=0.25$ for all datasets for comparison.  

\begin{figure}[tb]
	\centering
	\subfigure[Layer 1]{\includegraphics[width=0.35\columnwidth]{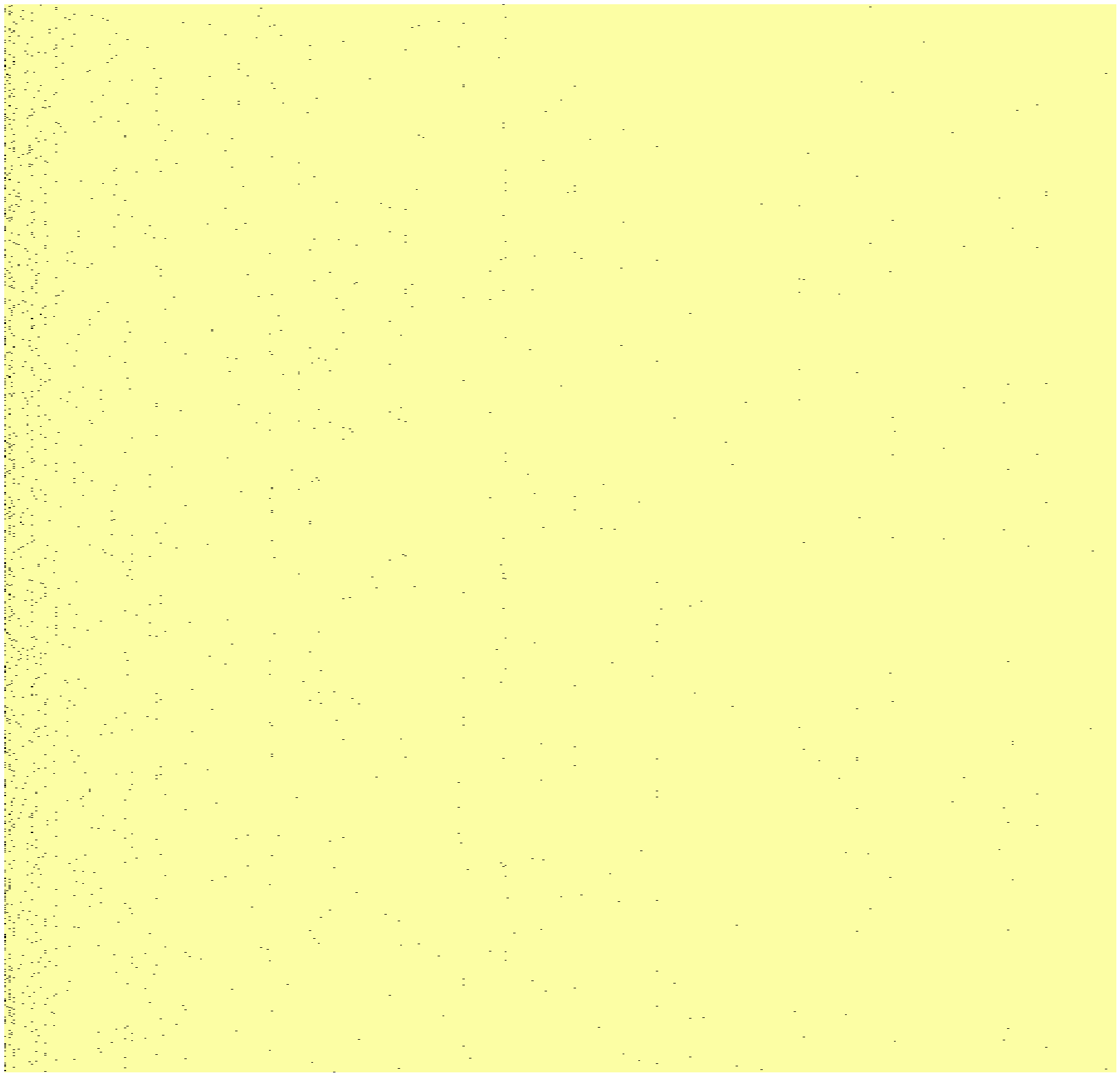}} \quad\quad\,\,
	\subfigure[Layer 2]{\includegraphics[width=0.35\columnwidth]{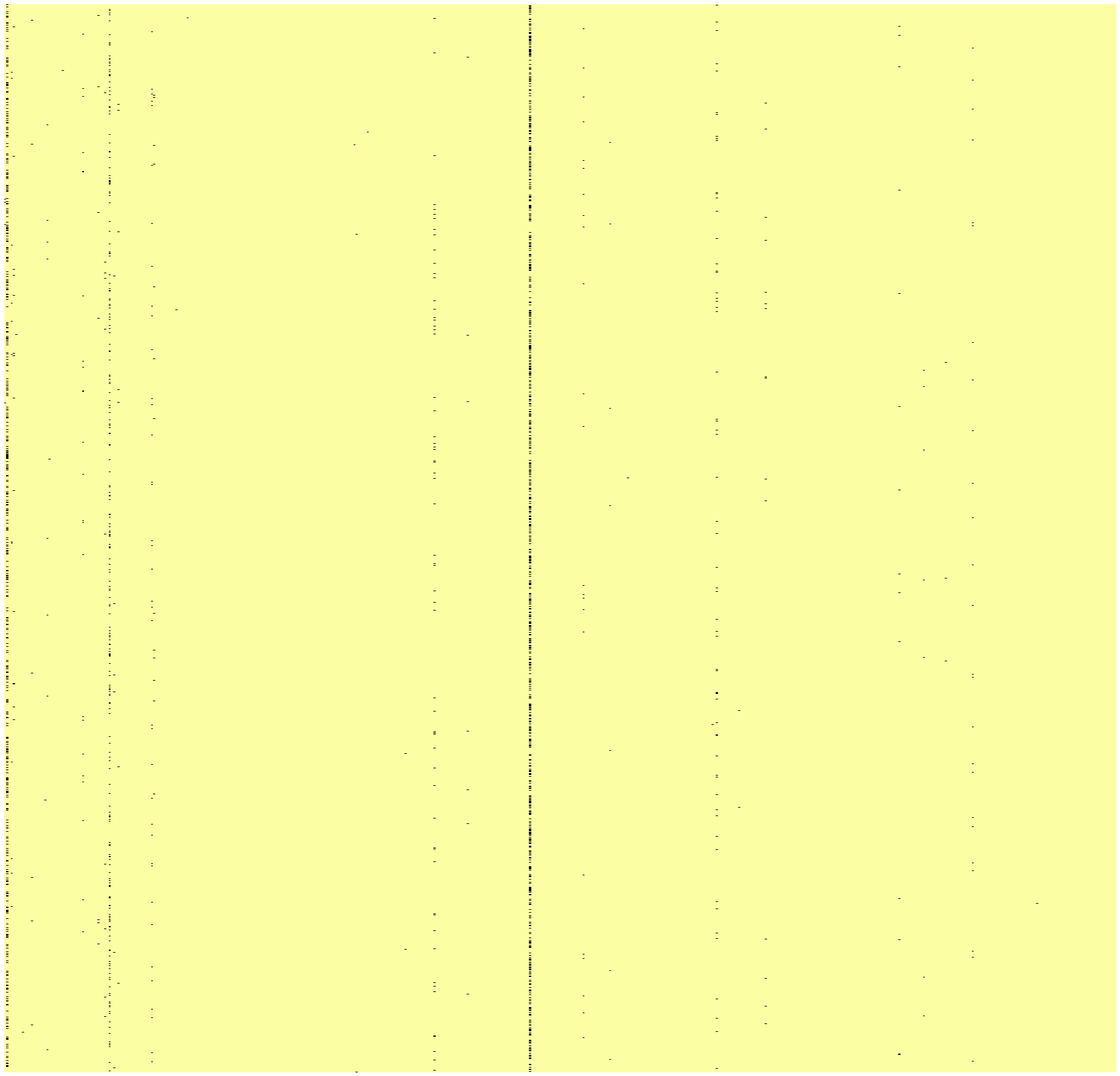}} 
	\caption{Visualization of Latent Graph Structure (given by two layers of NeuralGauSim) on CiteSeer dataset.}
	\label{graph3}
	\vspace{-1.5em}
\end{figure}

\section{Comparison of Complexity}
To compare the training time and GPU memory consumption, we statistic the training time (s) and GPU memory usage (MB) of the proposed method and previous methods, including IDGL, IDGL-Anchor, and NodeFormer, on CiteSeer dataset. Here, the number of transition nodes (anchor nodes for IDGL-Anchor) is set to be $500$. The experimental results are shown in Table \ref{tab4}. From this table, we have the following observations. First, we can observe that the proposed transition graph structure learning method significantly reduces the complexity of both time and memory consumption. Second, compared with the previous scalable methods, the proposed method exhibits comparable or even better efficiency.

\section{Visualization of Learned Graph Structure}
To investigate the learned graph structure of the proposed method, we visualize the graph structures learned by NeuralGauSim (given by two layers) on CiteSeer dataset, shown in Figure \ref{graph3}. From this figure, we can observe that different layers learn different graph structures, and this adaptive structure learning method holds potential utility in facilitating downstream classification tasks.   

\section{Conclusion}
In this paper, we focus on the differential graph structure learning framework and analyze the issue of structure sampling strategy. To fulfill the requirement of the edge sampling, we propose the Gaussian similarity modeling method and neural Gaussian similarity modeling method. To reduce the complexity, we develop a transition graph structure learning by transferring the initial nodes to transition nodes. Extensive experiments on graph and graph-enhanced application datasets demonstrate the effectiveness of the proposed methods. In future work, we aim to delve into effectiveness methods to tackle increasingly demanding scenarios, including heterogeneous graphs and multiplex graphs.

\section{Acknowledgment}
This work was supported in part by the National Natural Science Foundation of China under Grant 62206208, 62106185 and in part by the Fundamental Research Funds for the Central Universities under Grant XJSJ23021. 

\bibliography{reference}

\appendix
\setcounter{equation}{0}
\setcounter{figure}{0}
\setcounter{table}{0}

\section{Comparison with Gaussian Kernel}
Note that the previous work \cite{amgcn} utilizes the Gaussian kernel function as similarity measurement for graph structure learning. The modeling strategy of the Gaussian kernel is fundamentally different from the proposed method. Specifically, we present the Gaussian kernel function as follows
\begin{gather}
	\phi(x_i, x_j) = \exp(-\frac{\lVert x_i - x_j\rVert^2}{t})
\end{gather} 
where $t$ is the parameter of kernel function and $\lVert x_i - x_j\rVert^2$ is the square of Euclidean distance between node $v_i$ and $v_j$. Without loss of generality, we suppose the node representations are normalized using 2-norm normalization, and can obtain 
\begin{gather}
	\lVert x_i - x_j\rVert^2 = 2(1 - x_ix_j^{\top}).
\end{gather} 
Then, we can plot the relationship curve between sampling probability and Gaussian kernel similarity score in Figure \ref{graph_1}. From this figure, we can observe that the sampling probability and similarity of the edge $(v_i, v_j)$ also exhibit a linear relationship, wherein the greater the similarity value, the higher the probability of sampling that edge. The proposed methods can fulfill the requirement of the edge sampling probability, which entails an initial increase followed by a decrease as the similarity between node pairs diminishes. Here, we present the additional relation curves in Figure 2 to show the proposed neural Gaussian similarity modeling method.  

\begin{figure}[h]
	\centering
	\subfigure[$t=0.3$]{\includegraphics[width=0.326\columnwidth]{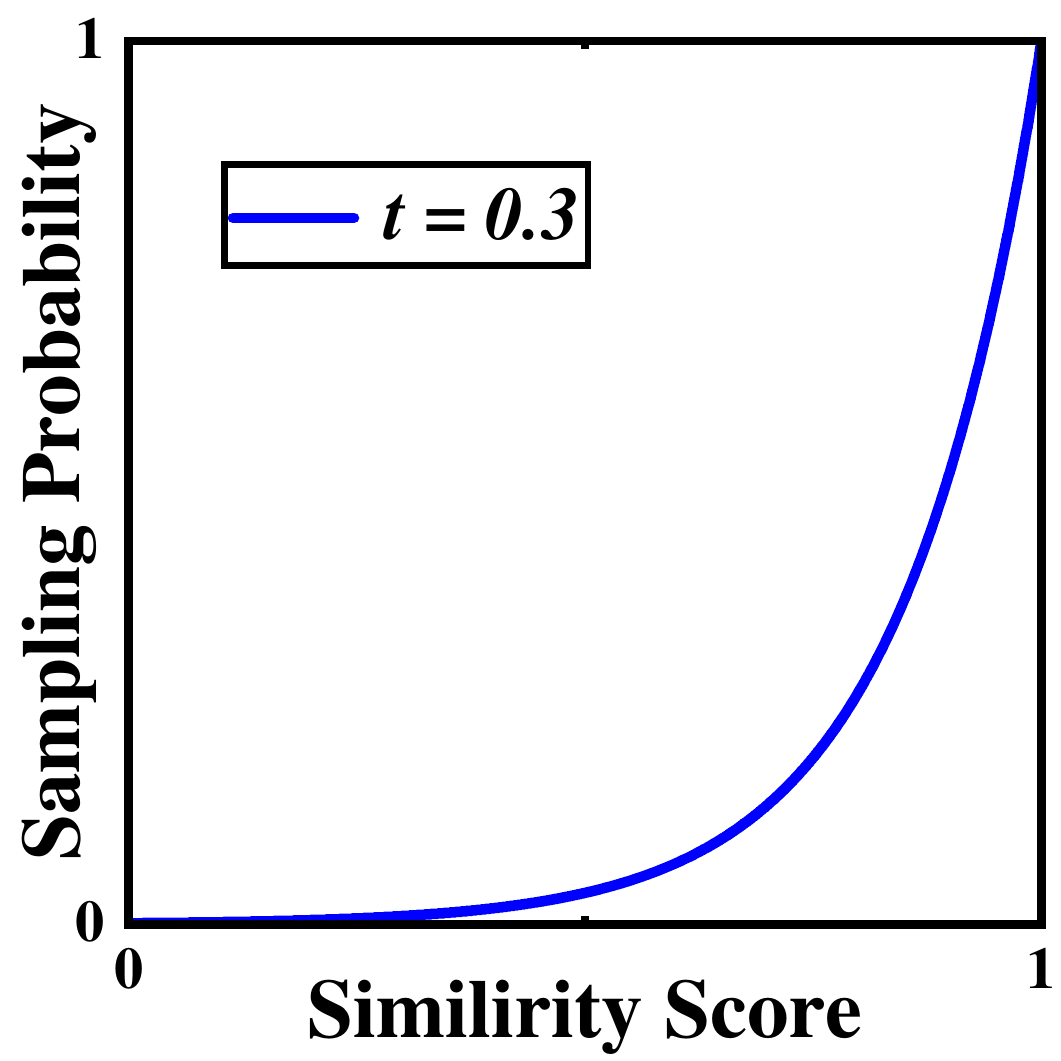}} 
	\subfigure[$t=0.5$]{\includegraphics[width=0.326\columnwidth]{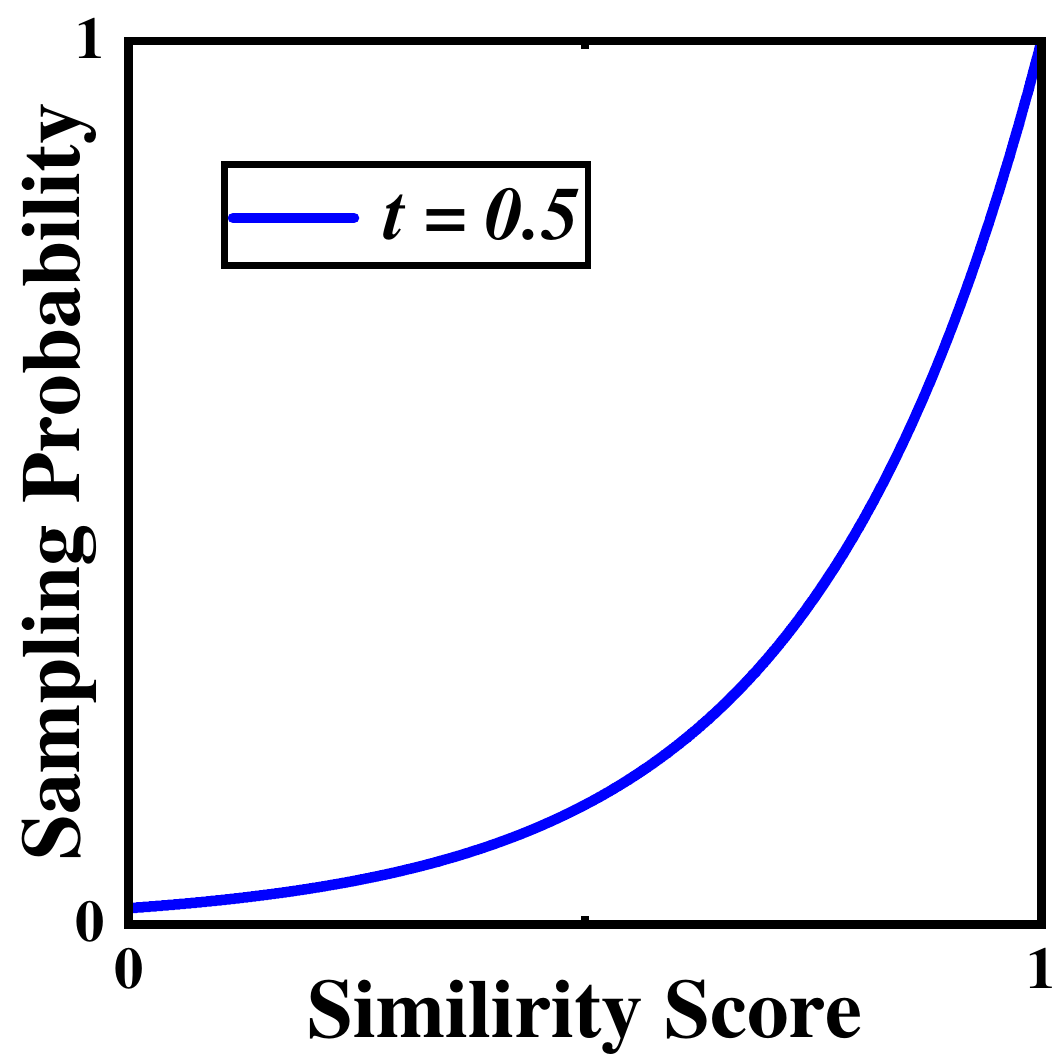}} 
	\subfigure[$t=1$]{\includegraphics[width=0.326\columnwidth]{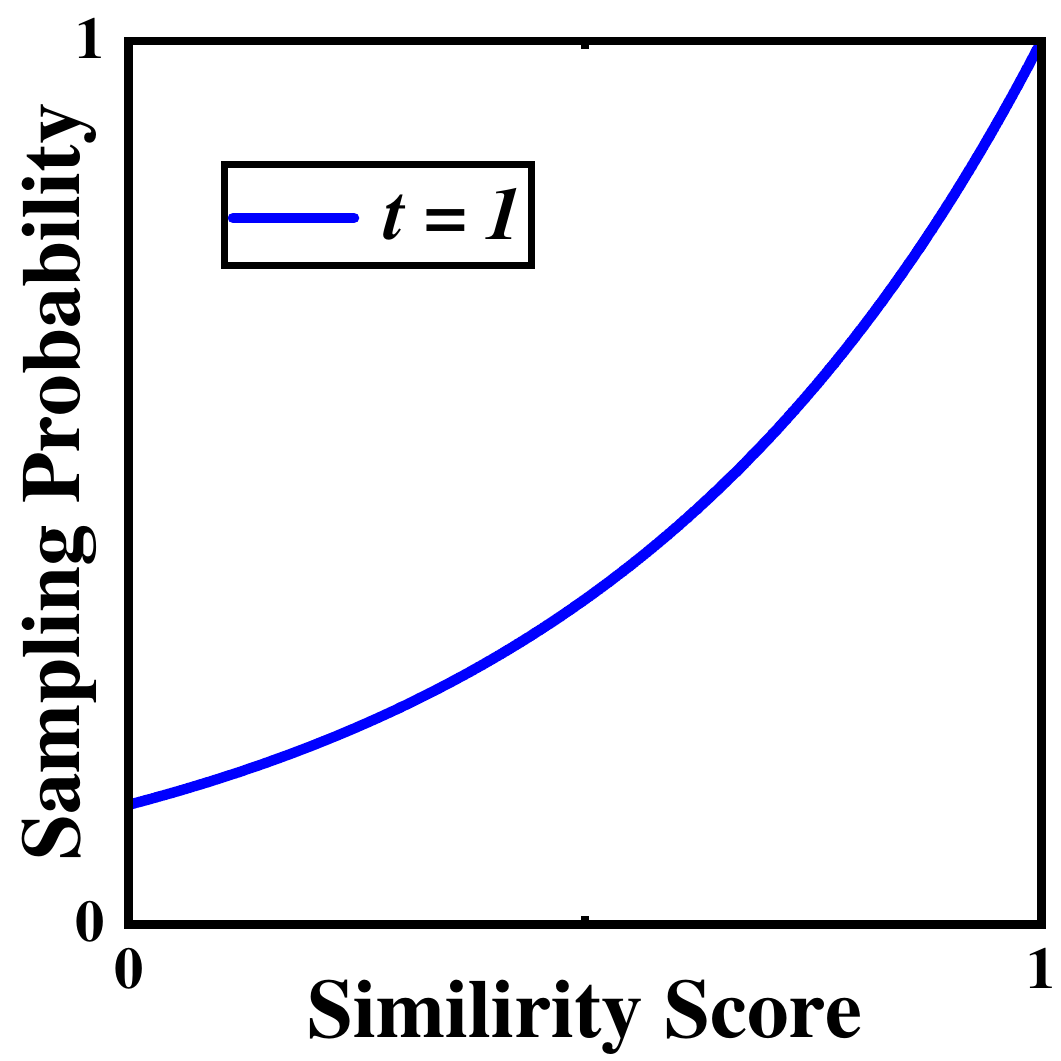}} 
	\caption{Relationship curve between sampling probability and Gaussian kernel similarity score.}
	\label{graph_1}
	\vspace{-1em}
\end{figure}

\begin{figure}[h]
	\centering
	\subfigure[GauSim]{\includegraphics[width=0.326\columnwidth]{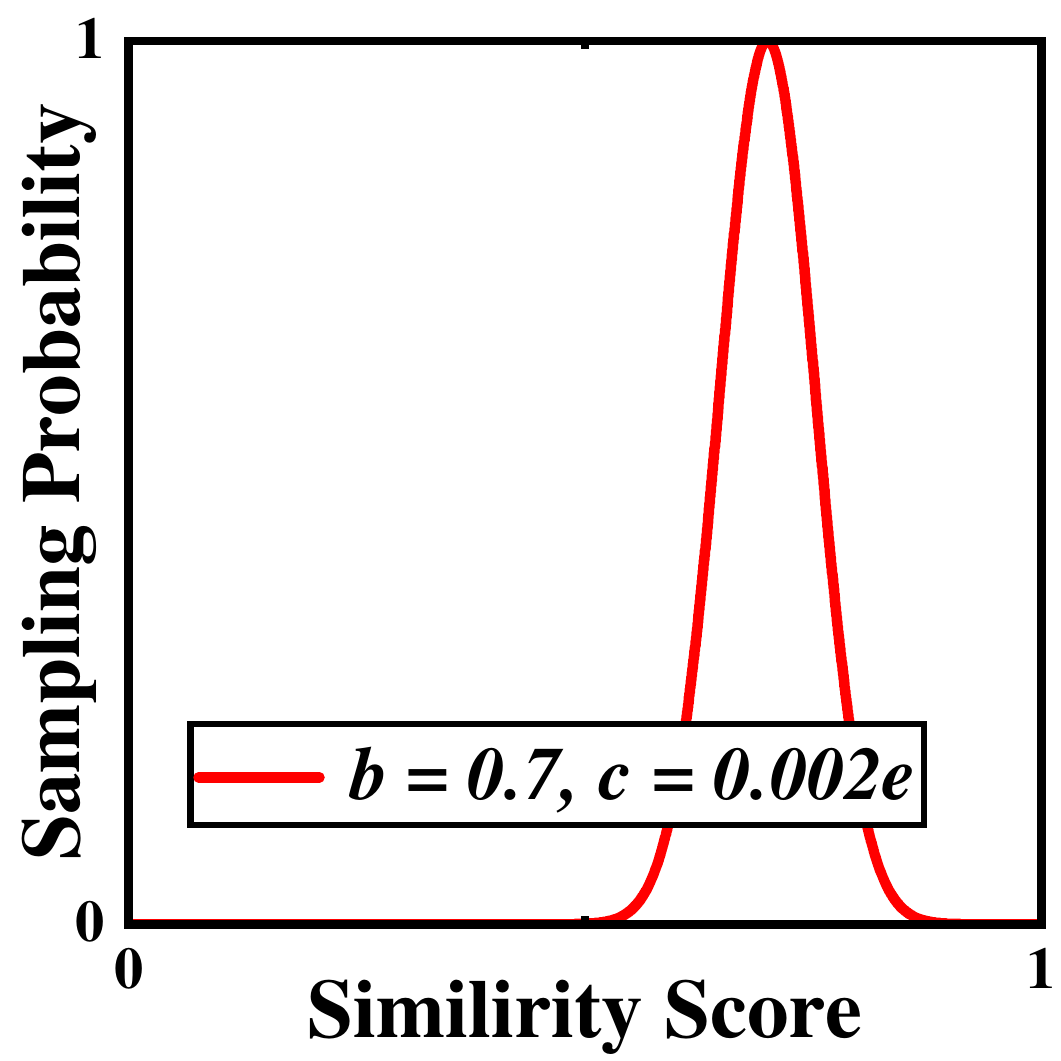}} 
	\subfigure[GauSim]{\includegraphics[width=0.326\columnwidth]{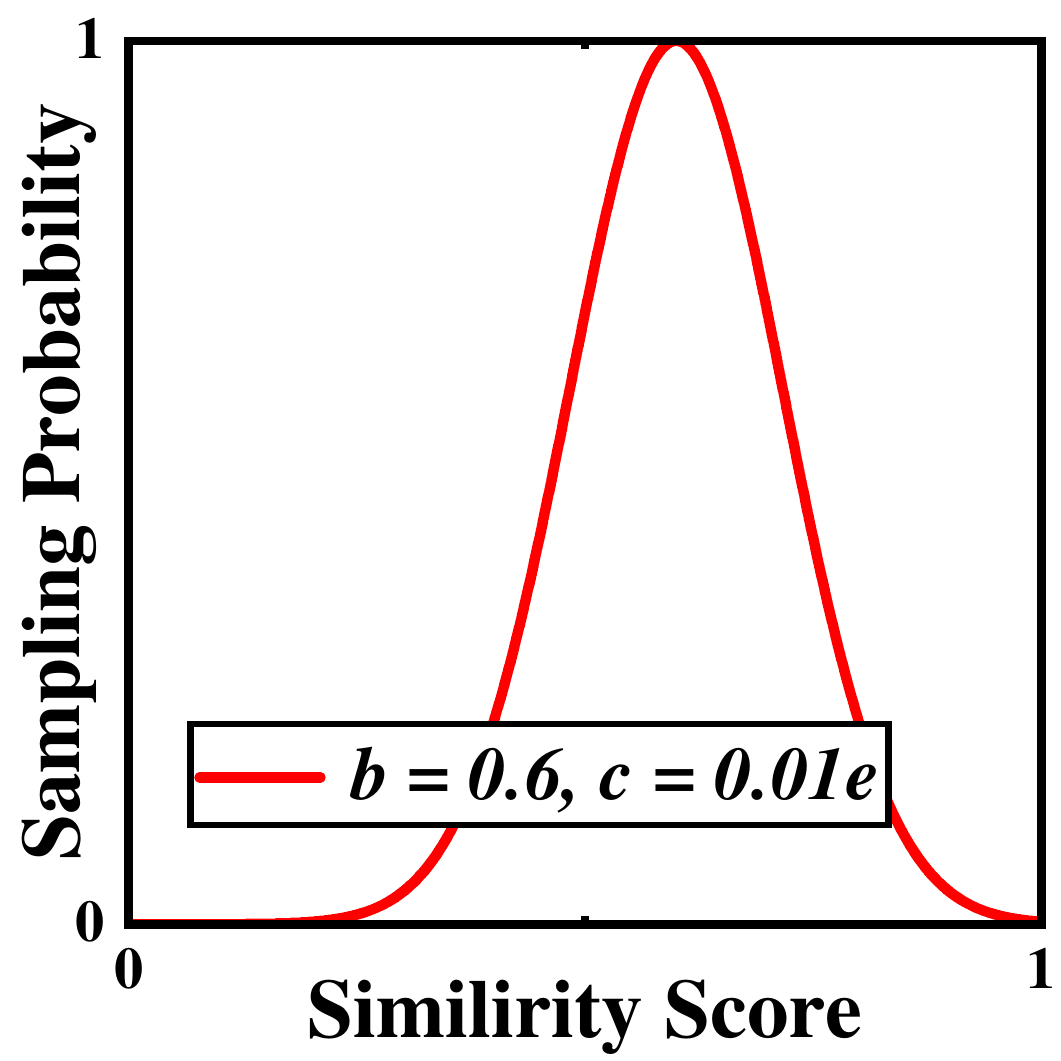}} 
	\subfigure[GauSim]{\includegraphics[width=0.326\columnwidth]{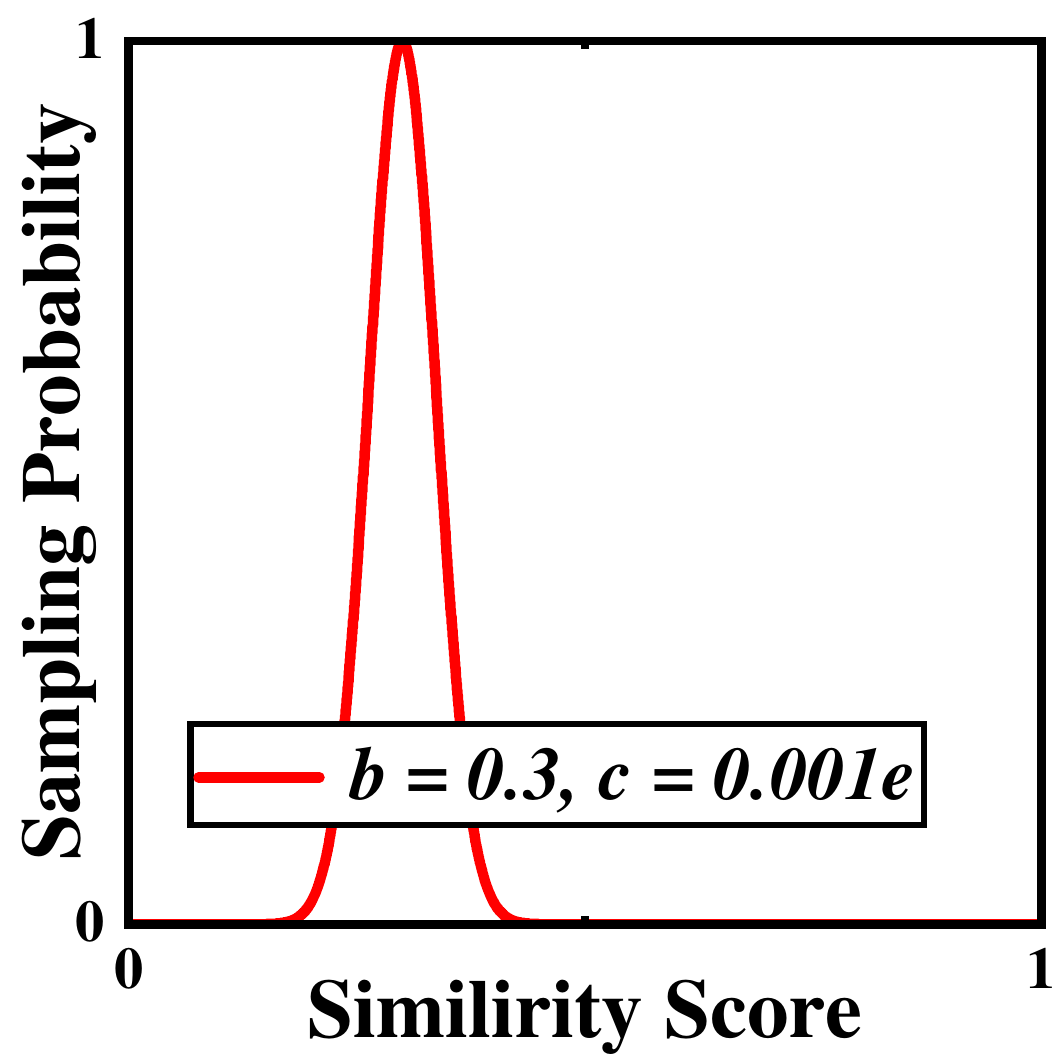}} 
	\caption{Relationship curve between sampling probability and the proposed Gaussian similarity similarity score.}
	\label{graph_2}
	\vspace{-1em}
\end{figure}

\section{Dataset Statistics} \label{ds}
We present detailed information for our used datasets in the experiment. The dataset statistics are shown in Table \ref{tab_1}. Specifically, CiteSeer and PubMed datasets are standard citation network benchmark datasets. In these networks, nodes represent papers, and the latent edges denote citations of one paper by another. Node features are the bag-of-words representation of papers, and node label is the academic topic of a paper. Chameleon and squirrel are two page-page networks on specific topics in Wikipedia. In these networks, nodes represent web pages and edges are mutual links between pages. Node features correspond to several informative nouns in the Wikipedia pages, and the nodes are classified into five categories in term of the number of the average monthly traffic of the web page. 
Coauthor CS and Coauthor Physics are co-authorship graphs. Here, nodes are authors, that are connected by an edge if they co-authored a paper; node features represent paper keywords for each author's papers, and class labels indicate most active fields of study for each author. 20News dataset is a collection of newsgroup documents, i.e. nodes, partitioned (nearly) evenly across 20 different
newsgroups. Mini-ImageNet dataset consists of $84\times 84$ RGB images from 100 different classes with 600 samples per class. Following \cite{nodeformer}, we choose 30 classes from the dataset, each with 600 images (nodes) that have 128 features extracted by CNN.   

\begin{table}[h]
	\centering
	\scalebox{0.76}{
		\begin{tabular}{|c|c|c|c|c|}
			\toprule
			\textsc{Dataset}&Context&\# Nodes&\# Features&\# Class\\
			\midrule
			\midrule
			CiteSeer&Citation network&3,327&3,703&6\\
			PubMed&Citation network&19,717&500&3\\
			Chameleon&Wikipedia network&2,277&2,325&5\\
			Squirrel&Wikipedia network&5,201&2,089&5\\
			CS&Coauthor network&18,333&6805&15\\
			Physics&Coauthor network&34,493&8,415&5\\
			20News&Text classification&9,607&236&10\\
			Mini-ImageNet&Image classification&18,000&128&30\\
			\bottomrule
	\end{tabular}}
	\caption{Statistics of graph benchmark datasets.}
	\label{tab_1}
	\vspace{-1em}
\end{table}

\section{Experimental Configuration}
All experiments are conducted with the following setting
\begin{itemize}
	\item Operating system: Ubuntu Linux release 20.04
	\item CPU: Intel(R) Xeon(R) Gold 5218 CPU @ 2.30GHz 
	\item GPU: NVIDIA GeForce RTX 3090 graphics card 
	\item Software version: Python 3.8, NumPy 1.20.1, Scipy 1.6.1,  PyTorch 1.11.0, PyTorch Geometric 2.0.4
\end{itemize}

\section{Current Limitations}
In this section, we discuss the limitations of the proposed model. First, some recent works develops several enhanced strategies, e.g., structure regularization and structure refinement. Incorporating these strategies may further improve the performance of the proposed model. Second, while this paper puts forth a scalable approach, addressing the issue of scaling to larger graphs, specifically those comprising millions of nodes, poses a more intricate challenge that requires further investigation and resolution in future research endeavors. Third, addressing the rising complexity of scenarios, encompassing heterogeneous graphs and multiplex graphs, emerges as a formidable challenge warranting further investigation. 

\end{document}